%% file: ms.tex
\documentclass[11pt]{article}

\pdfoutput=1
\usepackage[numbers]{natbib}

\usepackage[colorlinks=true,linkcolor=blue,citecolor=blue]{hyperref}
\usepackage[letterpaper, margin=1in]{geometry}

\title{Nonasymptotic Guarantees for Spiked Matrix Recovery \\with Generative Priors}

\usepackage[utf8]{inputenc}
\usepackage[T1]{fontenc}
\usepackage{authblk}

\usepackage{natbib}
\author[1]{Jorio ~Cocola}
\author[1,2]{Paul ~Hand\thanks{Partially supported by NSF CAREER Grant DMS-1848087 and NSF Grant DMS-2022205.}}
\author[3]{Vladislav ~Voroninski}
\affil[1]{Department of Mathematics, Northeastern University}
\affil[2]{Khoury College of Computer Sciences, Northeastern University}
\affil[3]{Helm.ai}

\input{preamble}

\begin{document}

\maketitle

\begin{abstract}
Many problems in statistics and machine learning require the reconstruction of a rank-one signal matrix from noisy data. Enforcing additional prior information on the rank-one component is often key to guaranteeing good recovery performance. One such prior on the low-rank component is sparsity, giving rise to the sparse principal component analysis problem. Unfortunately, there is strong evidence that this problem suffers from a computational-to-statistical gap, which may be fundamental. In this work, we study an alternative prior where the low-rank component is in the range of a trained generative network. We provide a non-asymptotic analysis with optimal sample complexity, up to logarithmic factors, for rank-one matrix recovery under an expansive-Gaussian network prior. Specifically, we establish a favorable global optimization landscape for a nonlinear least squares objective, provided the number of samples is on the order of the dimensionality of the input to the generative model. This result suggests that generative priors have no computational-to-statistical gap for structured rank-one matrix recovery in the finite data, nonasymptotic regime. We present this analysis in the case of both the Wishart and Wigner spiked matrix models.
\end{abstract}

\input{sections/intro.tex}
\input{sections/related_work.tex}

\input{sections/main_res.tex}
\input{sections/subgrad.tex}
\input{sections/num_expr.tex}


\bibliography{references}

\newpage
\appendix
\input{sections/supplement1.tex}
\input{sections/supplement2.tex}
\input{sections/supplement21.tex}
\input{sections/supplement3.tex}
\input{sections/supplement4.tex}

\end{document}

%% file: preamble.tex
\usepackage[utf8]{inputenc} 
\usepackage[T1]{fontenc}    
\usepackage{url}            
\usepackage{booktabs}       
\usepackage{amsfonts}       
\usepackage{nicefrac}       
\usepackage{microtype}      
\usepackage{color}
\newcommand{\red}[1]{{\color{red} {#1}}}

\usepackage{authblk}

\usepackage{amsfonts,amsmath,bbm,amssymb, amsthm}
\usepackage{comment}
 \usepackage[pdftex]{graphicx}
\usepackage{subfig}
\usepackage{enumitem}
\usepackage{subfig}
\usepackage{algorithm}

\usepackage[noend]{algpseudocode}

\newcommand{\htilde}{\tilde{h}}

\newcommand{\eps}{\epsilon}

\newcommand{\relu}{\text{relu}}
\newcommand{\T}{\intercal}
\newcommand{\bigO}{\mathcal{O}}
\newcommand{\R}{\mathbb{R}}
\newcommand{\diag}{\text{diag}}

\newcommand{\xstar}{{x_{\star}}}
\newcommand{\vbar}{{\bar{v}}}
\newcommand{\thbar}{{\bar{\theta}}}

\newcommand{\ystar}{{{y_\star}}}
\newcommand{\vtilde}{\tilde{v}}
\newcommand{\Sbeta}{\mathcal{S}_\beta}
\newcommand{\pa}{\partial}
\newcommand{\PX}{\mathbb{P}}
\newcommand{\EX}{\mathbb{P}}
\newcommand{\argmax}{\operatornamewithlimits{argmax}}

\newtheorem{assumption}{Assumption}
\newtheorem{prop}{Proposition}
\newtheorem{definition}{Definition}
\newtheorem{thm}{Theorem}
\newtheorem*{thm*}{Theorem}
\newtheorem{lemma}{Lemma}

\theoremstyle{remark}

%% file: sections/intro.tex
\section{Introduction}

In this paper we study the problem of estimating a spike vector $\ystar \in \R^n$ from data $Y$ consisting of a rank-1 matrix perturbed with random noise. In particular, the following random models for $Y$ will be considered.
\begin{itemize}
	\item The \textbf{Spiked Wishart Model} in which $Y \in \R^{N \times n}$ is given by:
	\begin{equation}\label{eq:wishart_Y}
	Y = u \, \ystar^\T + \sigma \mathcal{Z},
	\end{equation}
	where $\sigma > 0$, $u \sim \mathcal{N}(0,I_n)$ and $\mathcal{Z}$ are independent
	and $\mathcal{Z}_{ij}$ are i.i.d. from $\mathcal{N}(0,1)$.
	
	\item The \textbf{Spiked Wigner Model} in which $Y  \in \R^{n \times n}$ is given by:
	\begin{equation}\label{eq:wigner_Y}
	Y = \ystar \ystar^\T + \nu \mathcal{H} 
	\end{equation}
	where $\nu > 0$, $\mathcal{H} \in \R^{n \times n}$ 
	is drawn from a \textit{Gaussian Orthogonal Ensemble} GOE$(n)$, 
	i.e. $\mathcal{H}_{ii} \sim \mathcal{N}(0, 2/n)$ for all $1 \leq i \leq n$ and $\mathcal{H}_{ij} = \mathcal{H}_{ji} \sim \mathcal{N}(0,1/n)$ for $1 \leq j < i \leq n$.
\end{itemize}
Spiked random matrices have been extensively studied in recent years as they serve as a mathematical model for many statistical inverse problems such as PCA \citep{johnstone2001distribution,amini2008high, deshpande2014sparse, vu2012minimax}, synchronization over graphs \citep{abbe2014decoding,bandeira2015non,javanmard2016phase} and community detection  \citep{mcsherry2001spectral,deshpande2016asymptotic,moore2017computer}. They are, moreover, connected to the rank-1 case of other linear inverse problems such as matrix sensing and matrix completion under RIP-like assumptions on the measurements operator \cite{bhojanapalli2016global,zhang2018much}. 

In the high-dimensional/low signal-to-noise ratio 
regimes, it is fundamental to leverage additional prior information on the low-rank component in order to obtain consistent estimates of $\ystar$. 
Recent works, however,  have discovered that some priors give rise to gaps between what is statistically-theoretically optimal and can be achieved with unbounded computational resources, and what instead can be achieved with polynomial-time algorithms. A prominent example is represented by 
the Sparse PCA problem in which the vector $\ystar$ in \eqref{eq:wishart_Y} is taken to be sparse (see next section and \cite{bandeira2018notes,kunisky2019notes} for surveys of recent approaches).  

In this paper we study the spiked random matrix models  \eqref{eq:wishart_Y} and \eqref{eq:wigner_Y}, where 
the prior information on the planted signal $\ystar$ 
comes from a learned generative network. 
In particular, we assume that a generative neural network $G: \R^k \to \R^n$ with $k < n$, has been trained on a data set of spikes, and the unknown spike $\ystar \in \R^n$ lies on the range of $G$, i.e. we can write $\ystar = G(\xstar)$ for some $\xstar \in \R^k$.
As a mathematical model for the trained $G$, we consider  
a $d$-layer feed forward network of the form:
\begin{equation}\label{eq:Gx}
G(x) = \relu(W_d \dots \relu(W_2 \relu(W_1 x)) \dots)
\end{equation}
with weight matrices $W_i \in \R^{n_i \times n_{i-1}}$ and $\relu(x) = \max(x,0)$  is applied entrywise. We furthermore assume that the network is expansive, i.e. $n = n_d > n_{d-1} > \dots > n_0 = k$, and the weights have Gaussian entries. This  modeling assumption was introduced in \cite{HV17}, and additionally it and its variants were used in \cite{HHHV18,oscar2018phase,ma2018invertibility,hand2019global,song2019surfing}. See  Section \ref{sec:informal} for justifications of this model. 

Generative priors have been shown to close a computational-to-statistical gap in the Compressive Phase Retrieval problem. 
With a sparsity prior the information-theoretically optimal sample complexity is proportional to the sparsity level $s$ of the signal, on the other hand the best known algorithms (convex methods \cite{hand2016compressed,li2013sparse,ohlsson2011compressive}, iterative thresholding \cite{cai2016optimal, wang2017sparse, yuan2019phase}, etc.) require a sample complexity proportional to $s^2$ for stable recovery, a barrier which might not be resolvable by polynomial-time algorithms \cite{barbier2019optimal}.
Under the generative prior \eqref{eq:Gx}, \cite{oscar2018phase} has shown that, compressive phase retrieval is possible via gradient descent over a nonlinear objective with sample complexity proportional (up to log factors) to the underlying signal dimensionality $k$.
This result suggests that it may be possible to use generative priors to close other computational-to-statistical gaps such as for models \eqref{eq:wishart_Y} and \eqref{eq:wigner_Y}. 
Indeed, recently \cite{aubin2019spiked}
considered these low-rank models and the generative network prior \eqref{eq:Gx}
and shows that 
in the asymptotic limit $k,n,N \to \infty$ with $n/k = \bigO(1)$ 
and  $N/n = \bigO(1)$, an Approximate-Message Passing algorithm  achieves the statistical information-theoretic lower bound 
and no computational-to-statistical gap is present.  
%

This paper analyzes the low-rank matrix models \eqref{eq:wishart_Y} and \eqref{eq:wigner_Y} under the generative network prior \eqref{eq:Gx}. 

The contributions of this paper are as follows.  We
analyze the global landscape of a natural least-square loss over the range of the generative network demonstrating its benign optimization geometry. Our result provide further evidences for the claim that rank-one matrix recovery 
does not have  computational-to-statistical gaps when enforcing a generative prior in the non-asymptotic finite-data regime. This provides a second problem for which generative priors have closed such gaps in a non-asymptotic case.  
We further corroborate these findings by proposing a (sub)gradient algorithm which, as shown by our numerical experiments, is able to recover the sought spike with optimal sample complexity. 
This paper, therefore, strengthens the case for generative networks as priors for statistical inverse problems, not only because of their ability to learn natural signal priors, but also because of their capacity to lead to statistically optimal polynomial-time algorithms and zero computational-to-statistical gaps. 

\subsection{Problem formulation and main results}\label{sec:informal}

We consider the rank-one matrix recovery problem under a deep generative prior.  We assume that the signal spike lies in the range of the generative prior $\ystar = G(\xstar)$.  To estimate $\ystar$, we propose to 
first find an estimate $\hat{x}$ of the latent variable $\xstar$ and then use $G(\hat{x}) \approx \ystar$.
We thus consider the following  minimization problem\footnote{Under the  conditions below on the generative network, it was shown in \citep{HV17} that $G$ is invertible and therefore there exists a unique $\xstar$ that satisfies $\ystar = G(x)$.}:
\begin{equation}\label{eq:minM}
\min_{x \in \R^k} f(x) := \frac{1}{4} \| G(x)G(x)^\T - {M}\|_F^2.
\end{equation}
where:
\begin{itemize}
	\item for the \textbf{Wishart model} \eqref{eq:wishart_Y} we take $M =  \Sigma_N - \sigma^2 I_n$ with $\Sigma_N = Y^\T Y/N$.
	
	\item for the \textbf{Wigner model} \eqref{eq:wigner_Y} we take $M = Y$.
\end{itemize}
Despite the objective function \eqref{eq:minM} being nonconvex and nonsmooth, we show that it enjoys a favorable global optimization geometry for Gaussian weight matrices $\{W_i\}_{i=1}^d$. 
The informal version of our main results for the two spiked models is given below.
\begin{thm}[Informal]\label{thm:informal}
	Let $\ystar = G(\xstar)$ for a given a generative network $G: \R^k \to \R^n$ as in \eqref{eq:Gx}.
	Assume that each layer is sufficiently expansive, i.e.  $n_{i+1} = \Omega(n_i \log n_i)$, and the weights are Gaussian. Consider the minimization problem  \eqref{eq:minM} and assume 
	that up to factors dependent on the number of layers $d$:
	\begin{itemize}
		\item  for the \textbf{Wishart model}:
		$\sqrt{k \log n \, /N} \lesssim 1 $,
		\item for the \textbf{Wigner model}: 
		$\nu \sqrt{k \log n \, /n} \lesssim 1$.
	\end{itemize}
	With high probability:
	\begin{enumerate}[label=\Alph*.]
		\item for any nonzero point $x \in \R^k$ outside two small neighborhoods of $\xstar$ and $- \rho_d \xstar$ with $0 < \rho_d  \leq 1$, the objective function \eqref{eq:minM} has a direction of strict descent given almost everywhere by the gradient of $f$;
		
		\item the objective function values near
		$-\rho_d \xstar$ are larger than those near $\xstar$, while $x = 0$ is a local maximum;
		
		\item for any point $x$ in the small neighborhood around of $\xstar$, up to polynomials in $d$:
		\begin{itemize}
			\item for the \textbf{Wishart model}:	
			\begin{equation}\label{eq:ws_scaling}
			\| G(x) - \ystar\|_2	\lesssim \sqrt{\frac{k \log n }{N}},
			\end{equation}
			\item  for the \textbf{Wigner model}: 
			\begin{equation}\label{eq:wg_scaling}
			\| G(x) - \ystar\|_2	\lesssim \nu \sqrt{\frac{k \log n}{n}}.
			\end{equation}
		\end{itemize}
	\end{enumerate}
\end{thm}

Our main result characterizes the global optimization geometry of the problem \eqref{eq:minM} for a network with an expansive architecture and Gaussian weights. 
Even though the objective function in \eqref{eq:minM} is a piecewise-quartic polynomial, we show that outside two small neighborhoods around $\xstar$ and a negative multiple of it, there are no other spurious local minima or saddles, and every nonzero point has a strict linear descent direction. The point $x = 0$ is a local maximum and a neighborhood around $\xstar$ contains the global minimum of $f$.

We note, moreover, that for any point $x$ in the ``benign neighborhood'' of $\xstar$, the reconstruction error $\| G(x) - \ystar \|$ has information-theoretically optimal rates \eqref{eq:ws_scaling} and \eqref{eq:wg_scaling} corresponding (up to $\log$ factors) to the best achievable even in the simple case of a $k$-dimensional subspace prior. This implies that for the Wishart model the number of samples required to estimate $\ystar$ scales like the latent dimension $k$ which corresponds to the intrinsic degrees of freedom of the signal $\ystar$. Similarly for the Wigner model this implies that enforcing the generative network prior leads to a reduction of the noise by a factor of ${k}/{n}$.

Furthermore we observe that the direction of descent guaranteed by the theorem are almost everywhere given by the gradient of the objective function $f$. Our result, therefore, suggests 
that spiked matrix recovery
with a deep (random) generative network prior can be solved rate-optimally by simply minimizing over the range of the network via simple and computationally tractable algorithms such as gradient descent methods. For small enough step sizes, the iterates of these methods would converge to  one of the two neighborhoods where the gradients are small (identified in Theorem \ref{thm:informal}A), and avoiding the bad neighborhood of $- \rho_d \xstar$ can be done by exploiting the knowledge of the properties of the loss function (described in Theorem \ref{thm:informal}B) as done in Algorithm \ref{alg:subgrad} below and shown in the numerical experiments.
Finally, proving a convexity-like property of the ``benign neighborhood'' around $\xstar$ would ensure that the iterates will remain in this neighborhood and gradient descent will converge to a point with optimal error-rates \eqref{eq:ws_scaling} and \eqref{eq:wg_scaling}.
Formally proving the optimality and polynomial runtime of a gradient method for spiked matrix recovery is left for future work.

Regarding the Gaussian weight assumption, we observe that there is empirical evidence that the distribution of the weights of deep neural networks have properties consistent with those of Gaussian matrices \citep{arora2015deep}. Moreover, these observations have been used in advancing the theoretical understanding of deep network trained in supervised setting and in particular their ability to preserve the metric structure of the data \citep{giryes2016deep}. The randomness 
assumption has been further used by \citep{arora2014provable} to show that autoencoders with random weights can be learned in polynomial time. 
More recently, a series of works (see for example \citep{li2018learning, du2018gradient, oymak2019towards, mei2019generalization, chizat2019lazy}) have been dedicated to theoretical guarantees for training deep neural networks in the close-to-random regime of the \textit{Neural Tangent Kernel} \citep{jacot2018neural}. 
Finally, as for the case of compressed sensing in which the analysis of the random setting has led to considerable understanding 
of the problem as well as tangible practical innovations, we hope that the analysis of the random setting for deep generative networks will provide 
insights and generate novel developments in the field of statistical inverse problems.

We finally observe that signal recovery problems where multiple signal structures hold simultaneously,
e.g. low-rankness and sparsity, have been notoriously difficult, leading to no tractable algorithms at optimal sample complexity (see the next section for further details). Consequently, one might expect that enforcing low-rankness and generative priors would be comparably hard. In this work, we show instead that this combination of structural priors is not inherently difficult. 
This would motivate practitioners to invest in building and using generative priors, as those studied in this paper, in contexts where other priors have been traditionally used with suboptimal theoretical guarantees or empirical performance. 

%% file: sections/related_work.tex
\section{Related work}

\subsubsection*{Sparse PCA and other computational-to-statistical  gaps.}

Given a large number of samples data  $\{y_i\}^{N}_{i=1} \in \R^n$
the important statistical task of finding 
the directions that explain most of the variance (principal components)
is classically solved by PCA.  Insights on the statistical performance of this algorithm 
can be gained by studying spiked covariance models \citep{johnstone2001distribution}.
Under this model it is assumed that the data are of the form: 
\begin{equation}\label{eq:spiked_samples}
y_{i} = u_{i} \ystar + \sigma z_{i}
\end{equation}
where $\sigma > 0$,  $u_i \sim \mathcal{N}(0,1)$ and $z_{i} \sim \mathcal{N}(0, I_n)$ 
are independent and identically distributed, and $\ystar$ is the unit norm planted spike. Note that a matrix $Y \in \R^{N \times n}$ with rows $\{y_{i}\}_i$ can be written as \eqref{eq:wishart_Y}, and the $y_{i}$s are i.i.d. samples of $\mathcal{N}(0, \Sigma)$ 
where the population covariance matrix is $\Sigma = \ystar \ystar^\T + \sigma^2 I_n$.
Principal Component Analysis, then, estimates $\ystar$ via the leading eigenvector $\hat{y}$ 
of the empirical covariance matrix
$\Sigma_{N} = \frac{1}{N}\sum_{i=1}^N y_{i} y_{i}^\T.$
Standard techniques of high dimensional probability then show that as long as\footnote{We write $f(n) \gtrsim g(n)$ if $f(n) \geq C n$ for 
	some constant $C > 0$ that might depend $\sigma$ and $\|\ystar\|^2$. Similarly for $f(n) \lesssim g(n)$.} $N \gtrsim n$, with overwhelming probability:
\begin{equation}\label{eq:PCA_est}
\min_{\eps = \pm} \| \eps \hat{y} - \ystar \|_2 \lesssim \sqrt{\frac{n}{N}}.
\end{equation}

However, in the modern high dimensional data regime, it is not uncommon
to consider cases where the ambient dimension of the data $n$ 
is larger, or of the order, of the number of samples $N$. In this case,
bounds of the form \eqref{eq:PCA_est} become meaningless. Even worse,
in the asymptotic regime $ n/N \to c > 0$ and for $\sigma^2$ large enough, the spike $\ystar$ and the estimate $\hat{y}$ become orthogonal \citep{johnstone2009}. Moreover, minimax techniques 
can be used to show that in this regime	 no other estimators can achieve better overlap with $\ystar$ \citep{wainwright2019high}.

These negative results motivated the use of additional 
structural prior on the spike $\ystar$, aimed at reducing
the sample complexity of the problem. In recent years 
various priors has been studied such as
positivity \citep{montanari2015non}, cone constraints \citep{deshpande2014cone} 
and in particular sparsity \citep{johnstone2009}, \citep{zou2006sparse}. 
In the latter case $\ystar$
is assumed to be  $s$-sparse,
and it can be shown that for $N \gtrsim s \log n$
and $n \gtrsim s$,
then the $s$-sparse largest eigenvector $\hat{y}_s$ of $\Sigma_N$:
\[
\hat{y}_s = \argmax_{y \in \mathcal{S}_2^{n-1}, \|y\|_0 \leq s} y^\T \Sigma_N y
\]
satisfies the condition:
\[
\min_{\eps = \pm} \|\eps \hat{y}_s - \ystar \|_2 \lesssim \sqrt{\frac{s \log n}{N}}.
\]
In particular the number of samples must scale linearly 
with the intrinsic dimension $s$ of the signal. 
These rates are also minimax optimal, see for example 
\citep{vu2012minimax} for the mean squared error 
and \citep{amini2008high} for the support recovery.
Despite these encouraging results no known 
polynomial time algorithm is known that 
achieves such performances and for example 
the covariance thresholding algorithm of \cite{krauthgamer2015semidefinite}
requires $N \gtrsim s^2$ samples 
in order to obtain exact support recovery or estimation rate:
\[
\min_{\eps = \pm} \|\eps \hat{y}_s - \ystar \|_2 \lesssim \sqrt{\frac{s^2 \log n}{N}},	
\]
as shown in  \cite{deshpande2014sparse}. 
In summary, only computationally intractable algorithms 
are known to reach  the statistical limit $N =  \Omega(s)$ for Sparse PCA, 
while polynomial time methods are only sub-optimal requiring $N = \Omega(s^2)$. 
The study of this computational-to-statistical  gap was initiated by \cite{berthet2013computational} who investigated the detection problem via a reduction to the planted clique problem
which is conjectured to be computationally hard.

The hardness of sparse PCA has been further suggested 
in a series of recent works \citep{cai2013sparse,ma2015sum,lesieur2015phase,brennan2019optimal}. 
These works fit in the growing and important body of literature on computational-to-statistical gaps, which have also been found and studied in a variety of other contexts such 
as tensor principal component analysis \citep{richard2014statistical}, community detection \citep{decelle2011asymptotic} and synchronization over groups \citep{perry2018message}.
Many of these problems can be phrased as recovery a spike vector from a spiked random matrix models, and the hardness can then be viewed as arising from imposing 
simultaneously  low-rankness and additional prior information on the signal (sparsity in case of Sparse PCA).
This difficulty can be found in sparse phase retrieval as well, where \cite{li2013sparse}
has shown that for an $s$-sparse signal of dimension $n$ lifted to a rank-one matrix, 
$m = \bigO(s \log n)$ number of quadratic measurements are enough to ensure well-posedness, while $m \geq \bigO(s^2/\log^2 n)$ measurements are necessary 
for the success of natural convex relaxations of the problem.
Similarly \cite{oymak2015simultaneously}  studies the 
recovery of simultaneously low-rank and sparse matrices,
and show the existence of a gap between what can be achieved with convex and tractable relaxations and nonconvex and intractable methods.

\subsubsection*{Recovery with a generative network prior}

Recently, in the wake of successes of deep learning , deep generative networks have 
gained popularity as a novel approach for encoding  and 
enforcing priors. They have been
successfully used as a  prior for various statistical estimation problems 
such as compressed sensing \citep{bora2017compressed},
blind deconvolution \citep{asim2018solving}, inpainting \citep{yeh2016semantic},
and many more \citep{sonderby2016amortised,yang2017dagan,qiu2019robust,xue2018segan}, etc. 

Parallel to these empirical successes, a recent line of  works have investigated
theoretical guarantees for 
various statistical estimation tasks with generative network priors. 
Following the work of \cite{bora2017compressed},
\cite{hand2017global} have given global guarantees for compressed sensing,
followed then by many others for various inverse problems
\citep{shah2018solving,mixon2018sunlayer,hand2019global,aubin2019exact,qiu2019robust}. In particular \cite{hand2018phase} have shown 
that  $m = \Omega(k \log n)$ number of measurements
are sufficient to recover a signal from random phaseless observations, assuming 
that the signal is the output of a trained generative network with latent space of dimension $k$. 
Note that, contrary to the sparse phase retrieval problem, generative priors for phase retrieval allow optimal sample complexity, up to logarithmic factors, with respect to the intrinsic dimension of the signal. Further, when modeled by generative priors, that dimensionality could be much smaller than the sparsity level $s$ under a sparsity prior and an appropriate basis. 

Recently \cite{aubin2019spiked} has shown that when $\ystar$ is in the range of an expansive-Gaussian generative network with Relu activation functions,
then low-rank matrix recovery does not have a computational-to-statistical gap,
in the asymptotic limit $k,n,N \to \infty$ with $n/k = \bigO(1)$ 
and  $N/n = \bigO(1)$. They also provide a spectral algorithm and demonstrate that it is able to match asymptotically the information-theoretical optimal. These methods were then extended to the phase-retrieval problem in \cite{aubin2019exact}.

%% file: sections/main_res.tex
\section{Low-rank matrix recovery under a generative network prior}\label{sec:SpikedRes}
We are now ready to formulate our main theoretical result
for the spiked random matrix models. Its analysis will be based on
the following assumptions on the weights of the network.

\begin{assumption}\label{hyp:randW}
	The generative network $G$ defined in \eqref{eq:Gx}, has weights $W_i \in \R^{n_i \times n_{i-1}}$ with i.i.d. entries from $\mathcal{N}(0,1/n_i)$ and satisfying the expansivity  condition with constant $\eps >0$:
	\begin{equation*}
	n_{i+1} \geq c \eps^{-2} \log(1/\eps) n_{i} \log n_{i}
	\end{equation*}
	for all $i$ and a universal constant $c > 0$.
\end{assumption}

We remark that no assumption on the layer-wise independence of the weight matrices is required.

Due to the non-smoothness of the Relu activation functions, the generative network $G$ and the loss function $f$ are not differentiable everywhere. Therefore, following \cite{HHHV18}, we resort to some concepts from nonsmooth analysis \footnote{The reader is referred to \cite{clason2017nonsmooth} for more details.}. Since $f$ is continuous and piecewise smooth, at every point $x \in \R^k$,  $f$  has a \textit{Clarke subdifferential} given by:
\begin{equation}\label{eq:subdiff}
\pa f(x) = \text{conv} \{ v_1, v_2, \dots, v_T \}
\end{equation}
where conv denotes the convex hull of the vectors $v_1, \dots, v_T$, gradients of the $T$ smooth functions adjoint at $x$.  
In particular at a point where $f$ is differentiable $\pa f(x) = \{\nabla f(x) \}$. The terms subgradients will be used for the vectors $v_x \in \pa f(x)$.  

The next theorem will demonstrate the favorable optimization geometry of the minimization problem \eqref{eq:minM} for the spiked matrix models \eqref{eq:wishart_Y} and \eqref{eq:wigner_Y}. 
We will show that provided that the noise scales linearly with the latent dimension $k$, outside 0 and two small Euclidean balls around $\xstar$ and a negative multiple of $\xstar$, 
the subgradient $v_x$ give a descent direction
for the function $f(x)$. We let 
$\mathcal{B}(x,r)$ denotes
the Euclidean ball of radius $r$ around $x$, and $D_v f(x)$ denotes the (normalized) one-sided directional derivative of $f$ in direction $v$: $D_{v} f(x) = \lim_{t \to 0} \frac{f(x + t v) - f(x)}{t \|v\|_2}$.

\begin{thm}[Global Landscape Analysis]\label{thm:main_rand}
	Let Assumption \ref{hyp:randW} be satisfied with $\eps \leq K_1 d^{-96}$, consider the minimization problem \eqref{eq:minM} and assume that the noise variance ${\omega}$ satisfies ${\omega} \leq K_2 \|\xstar\|_2^2 2^{-d}/d^{44}$ where:
	\begin{itemize}
		\item for the \textbf{Spiked Wishart Model} \eqref{eq:wishart_Y}  take $M~=\Sigma_N-\sigma^2 I_n$ with $\Sigma_N = Y^\T Y/N$ and:
		\[
		{\omega} :={(\|\ystar\|_2^2 + \sigma^2)}  \max \Bigg\{\sqrt{\frac{113 k \log (3\, n_1^d n_2^{d-1} \dots n_{d-1}^2 n) }{N}}, \frac{52 k \log (3\, n_1^d n_2^{d-1} \dots n_{d-1}^2 n) }{N} \Bigg\};
		\]
		
		\item for the \textbf{Spiked Wigner Model} \eqref{eq:wigner_Y}  take $M = Y$ and:
		\[
		{\omega} := {\nu} \sqrt{  \frac{30 k \log (3\, n_1^d n_2^{d-1} \dots n_{d-1}^2 n)  }{n}}.
		\]
	\end{itemize}
	
	Then for $\gamma_\eps > 0$ depending polynomially on $\eps$,
	with probability at least  $1 - 2e^{- k \log n} - \sum_{i=1}^d 8 n_i e^{- \gamma_\eps n_{i-1}}$
	the following holds.
	
	For all $x \in \R^k$:
	\begin{itemize}
		\item if $x \notin \mathcal{B}(\xstar, r_+)~\cup~\mathcal{B}(-\rho_d \xstar, r_-)~\cup~\{0\}$ and $v_x \in \pa f(x)$:
		\[
		D_{-v_x} f(x) < 0
		\]
		where 
		\[
		r_+ = K_3 (d^{14} {\eps^{1/2}} + 2^d d^{10} \omega  \|\xstar\|_2^{-2})  \|\xstar\|_2,
		\]
		and
		\[
		r_- =  K_4 (d^{12} {\eps}^{1/4} + 2^{d/2} d^{10} \omega^{1/2} \|\xstar\|_2^{-1})   \|\xstar\|_2
		\]
		\item if $x \in \mathcal{B}(0,  \|\xstar\|_2/16 \pi) \backslash \{0\}$  and $v_x \in \pa f(x)$ then \[\langle x, v_x\rangle < 0\]
		while if $x = 0$ and $v \in \mathcal{S}^{k-1}$ then 
		\[
		D_{-v} f(0) = 0
		\]
	\end{itemize} 
	Here $K_1, \dots, K_4$  are universal constants and $0< \rho_d \leq 1$ depends only on the depth $d$ and converges to 1 as $d \to \infty$.
\end{thm}

According to the theorem the subgradients of $f$ are direction of strict descent for any nonzero point outside the two Euclidean balls around $\xstar$ and $-\xstar$. Furthermore, there are no other spurious critical points or non-escapable saddles apart from the local maximum $x = 0$.

For small enough $\eps > 0$, the size of the two neighborhoods around $\xstar$ and $-\xstar$ is a function of the control parameter $\omega$.
This quantity is analogous to the effective SNR $\sigma^2 \sqrt{s \log(n)/N}$
which governs sharp transitions in Sparse PCA \cite{amini2008high}.
In particular, for the Wishart model (structured PCA problem), in the interesting regime $k \log(n) \lesssim N$ and at fixed depth $d$, the size of the ball around $\xstar$ shrinks at the optimal rate $\sqrt{k \log(n)/N}$, implying that a number of samples $N$ proportional to $k$ is sufficient for a consistent estimate. In the same fashion, for the Wigner model the size of the noise $\nu$
it is required to be inversely proportional to the optimal rate $\sqrt{k/n}$.

We note that the quantity $2^d$ in the hypothesis and conclusions of the theorem, is an artifact of the scaling of the network and it should not be taken as requiring exponentially small noise. Indeed under the assumptions on the weights 
specified above, these matrices have spectral norm approximately 1, while the application of the Relu function zeros out approximately half of the entries of its argument leading to an ``effective'' operator norm of approximately $1/2$. The other polynomial dependence on the depth $d$ are likely not optimal and optimizing the proof for superior dependence on $d$ would not drastically alter the fundamental theoretical advance. As we show in the numerical experiments the bounds are quite conservative and the actual dependence on the depth is much better in practice.

Having analyzed the behavior of the loss $f$ outside the two balls around $x$ and $-\rho_d \xstar$, the next proposition 
will describe the local properties of these two neighborhoods.

\begin{prop}\label{prop:rand_localmin}
	Let the assumptions of Theorem  \ref{thm:main_rand} be satisfied.
	
	A. For any $x \in \mathcal{B}(\xstar, r_+)$ and $y \in \mathcal{B}(- \rho_d \, \xstar, r_-)$ it holds that:
	\[
	f(x) < f(y)
	\]
	
	B. In addition, for $K_5$ and $K_6$ positive absolute constants:
	\begin{itemize}
		\item for the \textbf{Spiked Wishart Model}, for all $x \in \mathcal{B}(\xstar, r_+)$:
		\[
		\| G(x) - \ystar\|_2	\leq K_5 \Big(d^4 {\eps^{1/2}} +  \frac{\omega}{\| \ystar\|_2^2 } \Big) d^{10} \|\ystar\|_2,
		\]
		
		\item  for the \textbf{Spiked Wigner Model}, for all $x \in \mathcal{B}(\xstar, r_+)$:
		\[
		\| G(x) - \ystar\|_2	\leq K_6 \Big(d^4 {\eps^{1/2}} + \frac{\omega}{\| \ystar\|_2^2 } \Big) d^{10} \|\ystar\|_2,
		\]
		
	\end{itemize}
	
\end{prop}

The previous results imply that, for $\eps$ small enough and under the assumptions on the noise level $\omega$, any point $x$ in the benign neighborhood $\mathcal{B}(\xstar, r_+)$ has reconstruction error $\| G(x) - \ystar\|$ which scales optimally according to \eqref{eq:ws_scaling} or \eqref{eq:wg_scaling}.

\subsection{Proofs outline and techniques}


The bulk of the analysis will be based on deterministic conditions on 
the weights of the network.
In particular we leverage a set of techinical results 
recently introduced by \cite{hand2017global}.

For $W \in \R^{n \times k}$ and $x \in \R^k$, define the operator $W_{+,x} := \diag(W x > 0) W$ such that $\relu(W x) = W_{+,x} x$.  Moreover let $W_{1,+,x} = (W_1)_{+,x} = \diag(W_1 x > 0) W_1$, and for $ 2 \leq i \leq d$:
\[
W_{i,+,x} = \diag(W_i, \Pi_{j=i-1}^{1} W_{j,+,x}  x > 0)W_i, 
\]
where $\Pi_{i=d}^{1} W_{i} = W_d W_{d-1} \dots W_1$. 
Finally we let $\Lambda_x=~\Pi_{j=d}^{1} W_{j,+,x} $
and note that $G(x) = \Lambda_x x$.

\begin{definition}[Weight Distribution Condition \citep{hand2017global}]\label{def:WDC}
	We say that $W \in \R^{n \times k}$ satisfies the \textbf{Weight Distribution Condition (WDC)} with constant $\epsilon > 0$ if for all $x_1, x_2 \in \R^k$:
	\[
	\| W_{+,x_1}^\T W_{+,x_2} - Q_{x_1,x_2}\|_2 \leq \eps,
	\] 
	where
	\[
	Q_{x_1,x_2} = \frac{\pi - \theta_{x_1,x_2}}{2 \pi} I_k + \frac{\sin \theta_{x_1,x_2}}{2 \pi} M_{\hat{x}_2\leftrightarrow \hat{x}_2}
	\]
	and $\theta_{x_1,x_2} = \angle (x_1,x_2)$, $\hat{x}_1 = x_1/\|x_1\|_2$, $\hat{x}_2 = x_2/\|x_2\|_2$, $I_k$ is the $k\times k$ identity matrix and $M_{\hat{x}_1\leftrightarrow \hat{x}_2}$ is the matrix that sends $\hat{x}_1 \mapsto \hat{x}_2$, $\hat{x}_2 \mapsto \hat{x}_1$, and with kernel span$(\{x_1,x_2\})^\perp$.
\end{definition}

Note that $Q_{x_1,x_2}$ is the expected value of $W_{+,x_1}^\T W_{+,x_2}$ when $W$ has rows $w_i \sim \mathcal{N}(0,I_k/n)$ and if $x_1 = x_2$ then $Q_{x_1,y_2}$ is an isometry up to the scaling factor $1/2$. 
This condition ensures that the angle between two 
vectors in the latent space is approximately preserved at the output layer and in turn guarantees the invertibility of the network. Under the Assumption \ref{hyp:randW}, \cite{hand2017global} shows that the WDC holds with high probability for all layers of the generative network $G$.

Next we observe that at a differentiable point the gradient of $f$, defined in \eqref{eq:minM}, is given by:
\begin{equation}\label{eq:gradx}
\nabla f(x) = \Lambda_x^\T [\Lambda_x x x^\T \Lambda_x^\T - M ] \Lambda_x x.
\end{equation}
Using the WDC, then, we demonstrate that $\nabla_x f(x)$ concentrates up to the noise level around a direction $h_x \in \R^k$ which is a continuous function of nonzero $x$ and $\xstar$.
Furthermore using the characterization \eqref{eq:subdiff} of the Clarke subdifferential,
we show that this concentration extends also at non-differentiable points for subgradients.

A direct analysis then shows that any directional derivative of $f$ at zero is zero and
that $h_x$ is small in a neighborhood 
of $\xstar$ and its negative multiple $-\rho_d \xstar$. This in turn guarantees the existence of a descent direction in the complement of these sets.

Similarly, we use the WDC to show
that up to noise level, the loss function $f$ concentrates around:
\[
f_E(x) = \frac{1}{4} \Big( \frac{1}{2^{2d}} \|x\|^4_2 +  \frac{1}{2^{2d}} \|\xstar\|^4_2  - 2 \langle x, \tilde{h}_x\rangle^2 \Big).
\]
where $\htilde_{x,\xstar}$ is  continuous for nonzero $x$ and $\xstar$. 
Directly analyzing the properties of $f_E$ in a neighborhood of $\xstar$ and $-\rho_d \xstar$ allows to derive the first point of Proposition \ref{prop:rand_localmin}. The last part of Proposition \ref{prop:rand_localmin} follows by noticing that the generator $G$ is locally Lipschitz.

Finally we extend a technique of \cite{HHHV18} to control 
the noise term: in our case a Gaussian Orthogonal matrix $\mathcal{H}$ for the spiked Wigner model \eqref{eq:wigner_Y}, and the differene between the empirical and the population covariance $\Sigma_N - \Sigma$ for the spiked Wishart model. 
The analysis is based on a counting argument on the subspaces spanned by a depth $d$ generative networks, which leads to the sought rate-optimal bounds.

%% file: sections/subgrad.tex
\section{A subgradient method and numerical experiments}\label{subsec:algo}

Informed by the analysis in the previous sections, we propose a gradient method for the solution of \eqref{eq:minM} and verify empirically its optimal properties. 

Recall that the main properties of the landscape of the minimization problem 
\eqref{eq:minM} were the global minimum in a neighborhood of the true latent vector $\xstar$ and a flat region in correspondence of $-\rho_d \xstar$. Moreover 
the latter region has larger loss function values than those in the vicinity of $\xstar$.  In order to overcome the non-convexity and avoid this bad region we run gradient descent from a random non-zero initialization and its negation. We then pick the iterate which has smaller final loss value.

\begin{algorithm}[H]
	\caption{Gradient method for the minizimization problem \eqref{eq:minM}}\label{alg:subgrad}
	\begin{algorithmic}[1]
		\State {\bfseries Input:} Weights $W_i$, observation matrix $M$, step size $\alpha > 0$, number of iterations $T$
		\State Choose $\hat{x}_0 \in \R^k \backslash \{0\}$ arbitrary
		\State Let $x_{0}^{(1)} = \hat{x}_0$ and $x_{0}^{(2)} = -\hat{x}_0$
		\For{i = 1, 2}
		\For{$j = 0, 1, \ldots, T-1$}
		\State Compute $v_{{x}_{j}}^{(i)} \in \pa f({x}_{j}^{(i)})$
		\State ${x}_{j+1}^{(i)} \gets {x}_{j}^{(i)} - \alpha v_{{x}_{j}}^{(i)}$;
		\EndFor
		\EndFor
		\If{$f({x}_{T}^{(1)}) < f({x}_{T}^{(2)})$}
		\State {Return:} {$x_{T}^{(1)}, G(x_{T}^{(1)})$}
		\Else
		\State {Return:} {$x_{T}^{(2)}, G(x_{T}^{(2)})$}
		\EndIf
	\end{algorithmic}
\end{algorithm}

Note that it will be highly unlikely that the iterates will be at a non-differentiable point, therefore 
in practice we can consider Algorithm \ref{alg:subgrad} with descent direction $v_{x} = \nabla f(x)$.

%% file: sections/num_expr.tex

We verify our theoretical claims 
on synthetic generative priors. We  consider  2-layer generative networks with Relu activation functions, hidden layer of dimension $n_1 = 250$, output dimension $n = 1700$ and varying number of latent dimension $k \in [10, 30, 70]$. 
We randomly sample the weights of the matrix independently from $\mathcal{N}(0, 2/n_i)$, which removes that $2^d$ dependence in Theorem \ref{thm:main_rand}.
We then consider data $Y$ according the spiked models \eqref{eq:wishart_Y} and \eqref{eq:wigner_Y}, where $\xstar \in \R^k$ is chosen so that $\ystar = G(\xstar)$ has unit norm. For the Wishart model we vary the samples $N$ while for the Wigner model we vary the noise level $\nu$ so that the following quantities remain constant for the different networks (latent dimension $k$):
\[
\theta_{\text{\tiny WS}}:=\sqrt{k\log(n_1^2\,n)/N}, \qquad  \theta_{\text{\tiny WG}}:=\nu \sqrt{k\log(n_1^2\,n)/n}
\]
We then plot the reconstruction error 
given by $\|G(x) - \ystar \|_2$ against $\theta_{\text{\tiny WS}} \approx  \sqrt{\frac{k}{N}}$ and $\theta_{\text{\tiny WG}} \approx \nu\sqrt{\frac{k}{n}}$.
As predicted by Theorem \ref{thm:main_rand} the errors scale linearly with respect to these control parameters, and moreover all the plots overlap confirming that these rates are tight with respect to the order of $k$.

\begin{figure}[h]%
	\centering
	\subfloat{{\includegraphics[width=6.5cm]{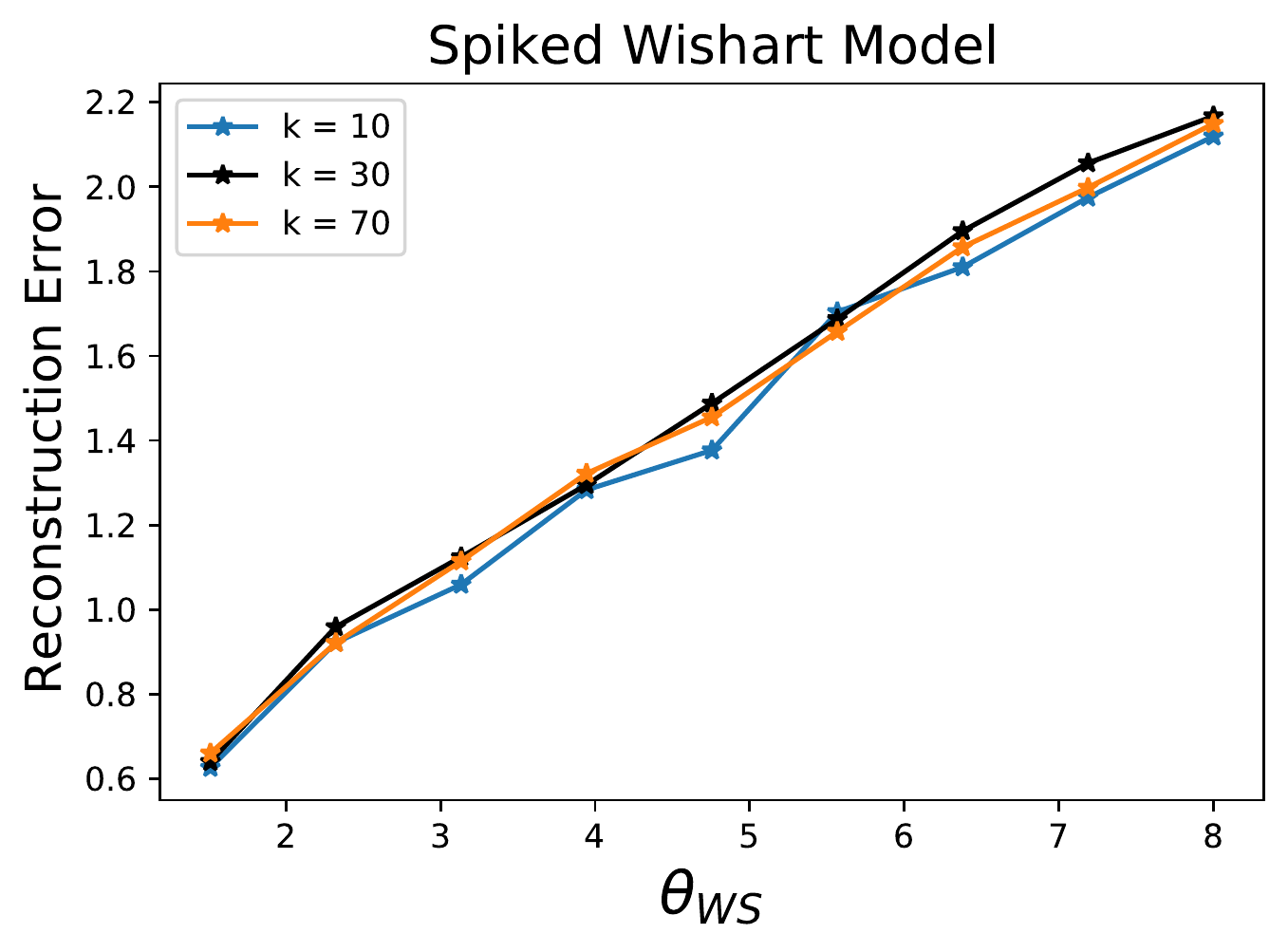} }}%
	\qquad
	\subfloat{{\includegraphics[width=6.5cm]{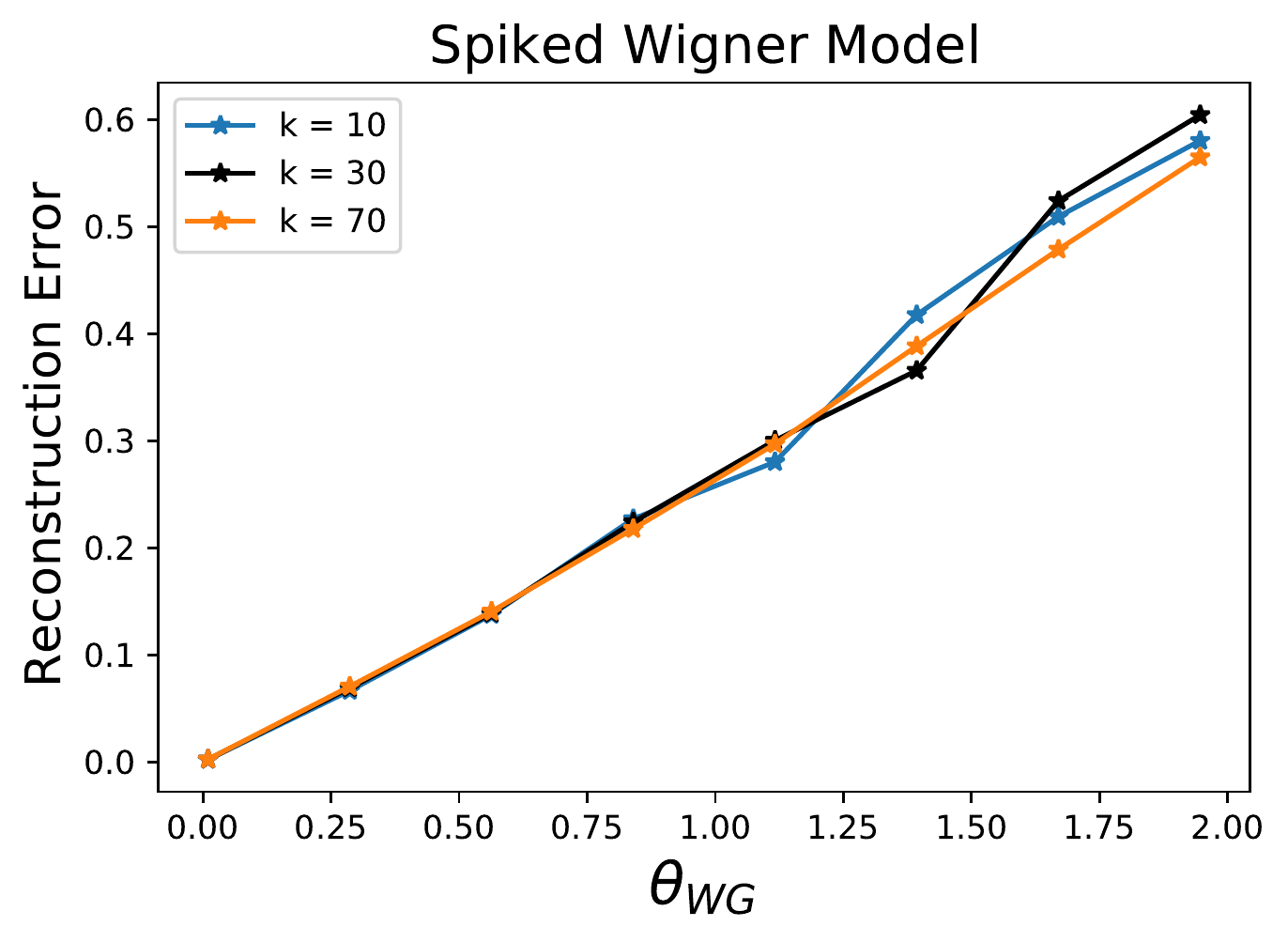} }}%
	\caption{Reconstruction error for the recovery of a spike $\ystar = G(\xstar)$ in Wishart and Wigner model with random generative network priors. Average over 50 random drawing of the network weights and samples. These plots demonstrate that the reconstruction error follow closely our theory.}%
	\label{fig:example}%
\end{figure}


%% file: sections/supplement1.tex
\section{Global landscape analysis under deterministic conditions}

As mentioned, the proof of Theorem \ref{thm:main_rand} and Proposition \ref{prop:rand_localmin}, will be based on deterministic conditions on 
the weights of the network and the noise matrix.
In particular we will consider the minimization problem \eqref{eq:minM} with: \[M = G(\xstar)G(\xstar)^\T + H,\] for an unknown symmetric matrix $H \in \R^{n \times n}$ and nonzero $\xstar \in \R^k$. 

Recall the definition \ref{def:WDC} of the WDC.
Below we will say that a $d$-layer generative network $G$ of the form \eqref{eq:Gx}, satisfies the WDC with constant $\eps>0$ if every weight matrix $W_i$ has the WDC with constant 
$\eps$ for all $i = 1, \dots d$.

We can now describe the landscape of the minimization problem \eqref{eq:minM} in a deterministic settings. 
\begin{thm}\label{thm:dir_deriv}
	
	Consider a generative network $G: \R^k \to \R^n$ as in \eqref{eq:Gx} and 
	the minimization problem \eqref{eq:minM} with unknown nonzero $\xstar$ and symmetric $H$.
	Fix $\eps > 0$ such that $K_1 d^{16} \sqrt{\eps} \leq 1$ and let $d \geq 2$. 
	Suppose that $G$ satisfies the WDC with constant $\eps$ and assume that:
	\begin{equation}\label{eq:noise}
	\|\Lambda_x^\T H \Lambda_x \|_2 \leq \frac{\omega}{2^d}  \qquad \text{for all}\; x \in \R^k, 
	\end{equation}
	with $2^d d^{12} w \leq K_2 \| \xstar \|^2$ and $K_2 < 1$. 
	
	Then for all $x \in \R^k$:
	\begin{itemize}
		\item if $x \notin \mathcal{B}(\xstar, r_+)~\cup~\mathcal{B}(-\rho_d \xstar, r_-)~\cup~\{0\}$ and \\$v_x \in \pa f(x)$:
		\[
		D_{-v_x} f(x) < 0
		\]
		where 
		\[
		r_+ = K_3 (d^4 {\eps^{1/2}} + 2^d \omega  \|\xstar\|_2^{-2}) d^{10} \|\xstar\|_2,
		\]
		and
		\[
		r_- =  K_4 (d^2 {\eps}^{1/4} + 2^{d/2} \omega^{1/2} \|\xstar\|_2^{-1}) d^{10}  \|\xstar\|
		\]
		\item if $x \in \mathcal{B}(0,  \|\xstar\|_2/16 \pi) \backslash \{0\}$  and $v_x \in \pa f(x)$ then 
		\[\langle x, v_x\rangle < 0\]
		while if $x = 0$ and $v \in \mathcal{S}^{k-1}$ then 
		\[D_{-v} f(0) = 0\]
	\end{itemize} 
	Here $\rho_d$ is a positive number that converges to 1 as $d \to \infty$ and $K_1, \dots, K_4$ are universal constants.
\end{thm}

Similarly, below we give the deterministic version of Proposition \ref{prop:rand_localmin}.
\begin{prop}\label{prop:local_min}
	Under the assumptions of Theorem \ref{thm:dir_deriv}, for any $\phi_d \in [\rho_d, 1]$, it holds that:
	\begin{equation}\label{eq:localmin}
	f(x) < f(y)
	\end{equation}
	for all $x \in \mathcal{B}(\phi_d \xstar, \varrho \|\xstar\| d^{-12})$ and $y \in \mathcal{B}(- \phi_d \xstar, \varrho  \|\xstar\| d^{-12})$ where $\varrho<1$ is a universal constant.
\end{prop}

The rest of the paper is organized as follows.
After summarizing the notation used throughout the paper in 
Section \ref{sec:Notation} and deriving concentration results for the subgradients from the WDC in Section \ref{subsec:prelim}, we give the proof of Theorem \ref{thm:dir_deriv} in Section \ref{sec:proof_det}. In Section
\ref{sec:proof_loc_min} we prove Proposition \ref{prop:local_min}, while Section
\ref{sec:supp_proof} contains the proofs of supplementary lemmas needed in the main results. 
Finally in Section \ref{sec:rand_proofs}, we derive the main Theorem \ref{thm:main_rand} and Proposition \ref{prop:local_min} from the corresponding deterministic ones by
controlling the noise terms and recalling a result of \cite{hand2017global} which shows that the WDC holds with high probability.

\subsection{Notation}\label{sec:Notation}

We now collect the notation that is used throughout the paper.  For any real number $a$, let $\relu(a) = \max(a, 0)$ and for any vector $v \in \R^n$, denote the entrywise application of $\relu$ as $\relu(v)$ . Let $\diag(W x > 0)$ be the diagonal matrix with $i$-th diagonal element equal to 1 if $(Wx)_i > 0$ and 0 otherwise.
For any vector $x$ we denote with $\|x\|$ its Euclidean norm 
and for any matrix  $A$ we denote with $\|A\|$ its spectral norm and with $\|A\|_F$ its Frobenius norm. The euclidean inner product between two vectors $a$ and $b$ is $\langle a,b \rangle$, while for 
two matrices $A$ and $B$ their Frobenius inner product will be denoted by $\langle A, B\rangle_F$.
For any nonzero vector $x \in \R^n$, let $\hat{x} = x/\|x\|$.  For a set $S$ we will write $|S|$ for its cardinality and $S^c$ for its complement.
Let $\mathcal{B}(x,r)$ be the Euclidean ball of radius $r$ centered at $x$, and $\mathcal{S}^{k-1}$ be the unit sphere in $\R^k$.
With $D_v f(x)$ we denote the (normalized) one-sided directional derivative of $f$ in direction $v$: $D_{v} f(x) = \lim_{t \to 0} \frac{f(x + t v) - f(x)}{t \|v\|_2}$.
We will write $\gamma = \Omega(\delta)$ to mean that 
there exists a positive constant $C$ such that $\gamma \geq C \delta$ and similarly $\gamma = \bigO(\delta)$ if $\gamma \leq C \delta$. Additionally we will use $a = b + O_1(\delta)$ when 
$\|a - b\| \leq \delta$, where the norm is understood to be the absolute value for scalars, the Euclidean norm for vectors and the spectral norm for matrices.


\subsection{Preliminaries}\label{subsec:prelim}

At a differentiable point, the gradient of $f$ is given by \eqref{eq:gradx} and will be denoted by $\vtilde_x$ and .
By the WDC, $\vtilde_x$ concentrates up to the 
noise level around the direction $h_x \in \R^k$:
\begin{equation} \label{eq:h_def}
h_x := \big[ \frac{1}{2^{2d}} x x^\T - \htilde_{x,\xstar} \htilde_{x,\xstar}^\T \big] x,
\end{equation}
where $\htilde_{x,\xstar}$ is defined below and
is a continuous function of $x$ and $\xstar$. The vector field  $\htilde_{x,\xstar}$ depends on a function that controls how the 
angles are contracted by the deep network, and defined as:
\begin{equation}\label{eq:def_g}
g(\theta) := \cos^{-1}\big( \frac{(\pi - \theta) \cos \theta + \sin \theta}{\pi} \big)
\end{equation}
With this definition we let $\htilde_{x,\xstar}$ be:
\[
\htilde_{x,\xstar} := \frac{1}{2^d} \big[ \big(\prod_{i=0}^{d-1} \frac{\pi - \bar{\theta}_i}{\pi} \big) \xstar + \sum_{i=1}^{d-1} \frac{\sin \bar{\theta}_i }{\pi}   \big(\prod_{j=i+1}^{d-1} \frac{\pi - \bar{\theta}_j}{\pi}\big) {\|\xstar\|} \hat{x} \big]
\]
where $\theta_i := g(\bar{\theta}_{i-1})$ for $g$ given by \eqref{eq:def_g} and $\theta_0 = \angle(x,y)$. For brevity of notation below we will use $\htilde_x = \htilde_{x,\xstar}$. For later convenience we also  define the following vectors:
\begin{equation*}
\begin{aligned}
p_x &:= \Lambda_x^\T \Lambda_x x; \\
q_x &:= \Lambda_x^\T \Lambda_\xstar  \xstar; \\
\vbar_x &:=  [p_x p_x^\T - q_x q_x^\T] \,x; \\
\eta_x &:= \Lambda_x^\T H \Lambda_x \, x.
\end{aligned}
\end{equation*}
and note that when $f$ is differentiable at $x$, then $\vtilde_x:= \nabla f(x) = \vbar_x - \eta_x$, in particular for zero noise $\vtilde_x = \vbar_x$.

We now observe the following facts.
\begin{lemma}[Lemma 8 in \cite{HV17}]\label{lemma:conctrWDC}
	Suppose that $d\geq 2$ and the WDC holds with $\eps < 1/(16 \pi d^2)^2$, then for all nonzero $x, \xstar \in \R^k$,
	\begin{align}
	\langle\Lambda_x x, \Lambda_\xstar \xstar \rangle &\geq \frac{1}{4 \pi} \frac{1}{2^d} \|x\|_2 \|\xstar\|, \label{eq:LxLy}\\
	\| \Lambda_x^\T \Lambda_\xstar \xstar - \htilde_{x,\xstar} \| &\leq 24 \frac{d^3 \sqrt{\eps}}{2^d} \|\xstar\|,  \; \text{and} \label{eq:htilde_conctr} \\
	\|\Lambda_x\|^2 \leq \frac{1}{2^d} (1 + 2 \eps)^d\leq \frac{1+ 4 \eps d}{2^d}  &\leq \frac{13}{12} \frac{1}{2^d} \label{eq:Lx_bound}. 
	\end{align}
\end{lemma}
\begin{proof}
	The first two bounds can be found in \citep[Lemma 8]{HV17}. The third bound follows noticing that the WDC implies:
	\[
	\|\Lambda_x \|^2 \leq \Pi_{i=d}^1 \| W_{i,+,x} \|^2 \leq \frac{1}{2^d} (1 + 2 \eps)^d \leq \frac{1 + 4 \eps d}{2^d} \leq \frac{13}{12}\frac{1}{2^d}
	\]
	where we used $\log(1+z) \leq z$ and $e^{z}\leq (1 + 2 z)$ for all $0\leq z \leq 1$.
\end{proof}

The next lemma shows that the noiseless gradient $\vbar_x$, concentrates around $h_x$.
\begin{lemma}\label{lemma:conctr_grad1}
	Suppose $d\geq 2$ and the WDC holds with $\eps < 1/(16 \pi d^2)^2$, then for all nonzero $x, \xstar \in \R^k$:
	\[
	\|\vbar_x - h_x \| \leq 86 \frac{ d^4 \sqrt{\eps}}{2^{2d}} \max(\|\xstar\|^2, \| x\|^2) \|x\|
	\]
\end{lemma}

We now use the characterization of the Clarke subdifferential given in \eqref{eq:subdiff}, to derive a bound on the concentration of $v_x \in \pa f(x)$ around $h_x$ up to the noise level.
\begin{lemma}\label{lemma:conctr_grad2}
	Under the assumption of Lemma \ref{lemma:conctr_grad1}, and with $H$ satisfying \eqref{eq:noise}, for any $v_x \in \pa f(x)$:
	\[
	\|v_x - h_x \| \leq 86 \frac{ d^4 \sqrt{\eps}}{2^{2d}} \max(\|\xstar\|^2, \| x\|^2) \|x\| +   \frac{\omega}{2^{d}} \,  \| x\|
	\]
\end{lemma}

%% file: sections/supplement2.tex
\subsection{Proof of Theorem \ref{thm:dir_deriv}}\label{sec:proof_det}

We define the set $\Sbeta$ outside which we can lower bound the gradient as:
\[
\Sbeta := \big\{ x \in \R^k |\;  \|h_x\| \leq \frac{\beta}{2^{2d}}  \max(\|x\|^2, \|\xstar\|^2) \|x\|  \big\}
\]
with:
\begin{equation}\label{eq:beta_def}
\beta = 5 \cdot \big(86 d^4 \sqrt{\eps} + 2^d \omega \|\xstar\|^{-2} \big)
\end{equation}
Outside the set $\Sbeta$ the gradient is bounded below and the landscape has
favorable optimization geometry. 

\noindent Due to the continuity and piecewise smoothness of the generator $G$ and in turn of the loss function $f$, for any $x,y \neq 0$ there exists a sequence of $\{x_n\} \to x$ such that $f$ is differentiable at each $x_n$ and $D_y f(x) = \lim_{n \to \infty} \nabla f(x_n) \cdot y$. It follows that:
\[
D_{- v_x} f(x) = - \lim_{n \to \infty} \vtilde_{x_n} \cdot \frac{v_x}{\|v_x\|}
\]
as $\nabla f(x_n) = \vtilde_{x_n}$.
Regarding the right hand side of the above, observe that: 
\[
\begin{aligned}
\vtilde_{x_n} \cdot v_x = \,&h_{x_n} \cdot h_x + (v_{x_n} - h_{x_n}) \cdot h_x 
+ h_{x_{n}} \cdot (v_x - h_x) + (\vtilde_{x_n} - h_{x_{n}}) \cdot (v_x - h_x)\\
\geq \,&h_{x_n} \cdot h_{x} - \|\vtilde_{x_n} - h_{x_n}\| \|h_x\| - \|h_{x_n}\| \| v_x - h_x\|
- \| \vtilde_{x_n} - h_{x_n}\| \| v_x - h_x\|,\\
\geq \,&h_{x_n} \cdot h_x 
- \frac{86 d^4 \sqrt{\eps}  + 2^d \omega \|\xstar\|^{-2}}{2^{2d}} \Big( \max(\|x_n\|^2, \|\xstar\|^2) \|x_n\| \|h_{x}\|
+ \max(\|x\|^2, \|\xstar\|^2) \|x\| \|h_{x_n}\| \Big)
\\ \,&- \Big(\frac{86 d^4 \sqrt{\eps} + 2^d \omega \|\xstar\|^{-2}}{2^{2d}}\Big)^2 \max(\|x\|^2, \|\xstar\|^2) \max(\|x_n\|^2, \|\xstar\|^2) \| x_n\| \|x\|
\end{aligned}	
\]
where the second inequality follows from Lemma \ref{lemma:conctr_grad2}. Moreover as $h_x$ is continuous in $x$ for all nonzero $x$:
\[
\begin{aligned}
\lim_{n \to \infty} \vtilde_{x_n} \cdot v_x 
\geq \,& \| h_x \|^2 
- 2 \frac{86 d^4 \sqrt{\eps}  + 2^d \omega \|\xstar\|^{-2}}{2^{2d}}  \max(\|x\|^2, \|\xstar\|^2) \|x\| \|h_{x}\|  \\
\;\;&- \Big(\frac{86 d^4 \sqrt{\eps} + 2^d \omega \|\xstar\|^{-2}}{2^{2d}}\Big)^2 \max(\|x\|^2, \|\xstar\|^2)^2  \|x\|^2 \\
\geq \,& \frac{\|h_x\|}{2} \Big[ \|h_x\| - 4  \frac{86 d^4 \sqrt{\eps}  + 2^d \omega \|\xstar\|^{-2}}{2^{2d}}  \max(\|x\|^2, \|\xstar\|^2) \|x\|  \Big] \\
\,&+ \frac{1}{2} \Big[ \|h_x\|^2 - 2 \Big(\frac{86 d^4 \sqrt{\eps} + 2^d \omega \|\xstar\|^{-2}}{2^{2d}}\Big)^2 \max(\|x\|^2, \|\xstar\|^2)^2  \|x\|^2  \Big]
\end{aligned}
\]
By our choice of $\beta$ in \eqref{eq:beta_def} it follows that for 
any  $x \in S_\beta^{c} \backslash \{0\}$ :
\[
\|h_x\| - 4  \frac{86 d^4 \sqrt{\eps}  + 2^d \omega \|\xstar\|^{-2}}{2^{2d}}  \max(\|x\|^2, \|\xstar\|^2) \|x\|  \geq \frac{\max(\|x\|^2, \|\xstar\|^2)}{2^{2d}} \|x\| \Big(  \beta- 4 \big(86 d^4 \sqrt{\eps} + 2^d \omega \|\xstar\|^{-2} \big)  \Big),
\]
so that:
\[
\lim_{n \to \infty} v_{x_n} \cdot v_x \geq \frac{\| h_x \|}{2} \frac{\max(\|x\|^2, \|\xstar\|^2)}{2^{2d}} 86 d^4 \sqrt{\eps} \|x \|  > 0.
\]
The latter equation allows to conclude $D_{-v_x} f(x) < 0$  for any nonzero $x \in S_\beta^{c}$ and any $v_x \in \pa f(x)$.
Finally observe that the radii of the neighborhoods around $\xstar$ and $-\rho_d \xstar$ can be found  
applying Lemma \ref{lemma:Sbeta_red} below with $\beta$ as given in \eqref{eq:beta_def}.

Next for any nonzero $v \in \R^{k}$ and $\tau \in \R$ we have:
\[
f(\tau v ) - f(0) = \frac{\tau^4}{4} \|G(v)G(v)^\T\|^2_F - \frac{\tau^2}{2} \langle G(v)G(v)^\T, G(\xstar)G(\xstar)^\T + H \rangle_F,
\] 
which implies that $D_{v} f(0) = 0$ for any $v \in \mathcal{S}^{k-1}$.

Finally notice that at a differentiable point $x \in \R^k$:
\[
\begin{aligned}
\langle \vtilde_x, x \rangle &= \langle \Lambda_x^\T [\Lambda_x x x^\T \Lambda_x^\T - \Lambda_\xstar \xstar \xstar^\T \Lambda_\xstar^\T  ] \Lambda_x x, x \rangle  - \langle \Lambda_x H \Lambda_x , x\rangle \\
&= \|G(x)\|^4 - \langle G(x), G(\xstar)\rangle^2 - \langle \Lambda_x H \Lambda_x , x\rangle \\
&\leq \frac{\|x\|^2}{2^{2d}} \Big[ \Big(\frac{13}{12} \Big)^2 \|x\|^2 - 
\Big(\frac{1}{16 \pi^2} - \frac{2^d \omega}{\|\xstar\|^2} \Big) \|\xstar\|^2 \Big]  \\
&\leq \frac{\|x\|^2}{2^{2d}} \Big[ 2 \|x\|^2 - 
\frac{1}{32 \pi^2} \|\xstar\|^2 \Big]
\end{aligned}
\]
having used \eqref{eq:LxLy}, \eqref{eq:Lx_bound} and the assumption on the noise \eqref{eq:noise} in the first inequality and $2^d d^{12} w \leq K_2 \| \xstar \|^2$ with $d\geq 2$ in the last one. We conclude that if $f$ is differentiable at $x \in \mathcal{B}(0, \|\xstar\|/16 \pi)$ then $\langle x, \vtilde_x \rangle < 0$.

If $f$ is not differentiable at a nonzero  $x \in \mathcal{B}(0, \|\xstar\|/16 \pi)$, then  by \eqref{eq:subdiff} for any $v_x \in \pa f(x)$:
\[
\begin{aligned}
\langle v_x, x \rangle &=   \langle c_1 v_1 + c_2 v_2 + \dots + c_T v_T, x \rangle  \\
&\leq (c_1 + c_2 \dots +c_T) \frac{\|x\|^2}{2^{2d}} \Big[ 2 \|x\|^2 -  \frac{1}{32 \pi^2} \|\xstar\|^2 \Big] < 0
\end{aligned}
\]

\subsubsection{Control  of the zeros of $h_x$}

In this section we show that $h_x$ is nonzero outside two 
neighborhoods of $\xstar$ and $-\rho_d \xstar$.

\begin{lemma}\label{lemma:Sbeta_red}
	Suppose $8 \pi d^6 \sqrt{\beta} \leq 1$. Define:
	\[
	\rho_d := \sum_{i=0}^{d-1} \frac{\sin \check{\theta}_i }{\pi}   \big(\prod_{j=i+1}^{d-1} \frac{\pi - \check{\theta}_j}{\pi} \big)
	\]
	where $\check{\theta}_0 = \pi$ and $\check{\theta}_i = g(\check{\theta}_{i-1})$. If $x \in \Sbeta$, 
	then we have either:
	\[
	|\bar{\theta}_0| \leq 32 d^4 \pi \beta \quad \text{and} \quad |\|x\|^2 - \|\xstar\|^2 | \leq 258 \pi  \beta d^6 \|\xstar\|
	\]
	or
	\[
	|\bar{\theta}_0 - \pi| \leq 8 \pi d^4 \sqrt{\beta} \quad \text{and} \quad |\|x\|^2 - \rho_d^2 \|\xstar\|^2 | \leq 281 \pi^2 \sqrt{\beta} d^{10} \|\xstar\|.
	\]
	In particular, we have:
	\[
	\Sbeta \subset \mathcal{B}(\xstar, R_1  \beta d^{10} \|\xstar\|) \cup  \mathcal{B}(-\rho_d \xstar, R_2 \sqrt{\beta} d^{10} \|\xstar\|) 
	\]
	where $R_1, R_2$ are numerical constants and $\rho_d \to 1$ as $d \to \infty$.
\end{lemma}

\begin{proof}
	Without loss of generality, let $\xstar = e_1$ and $x = r \cos \thbar_0 \cdot e_1 + r \sin \thbar_0 \cdot e_2$, 
	for some $\thbar_0 \in [0,\pi]$,  and $r \geq 0$. Recall that we call $\hat{x} = x/\|x\|$ and $\hat{x}_\star = \xstar/\|\xstar\|$.
	We then introduce the following notation:
	\begin{equation}\label{eq:zeta_xi_def}
	\xi = \prod_{i=0}^{d-1} \frac{\pi - \thbar_i}{\pi}, \quad \zeta = \sum_{i=0}^{d-1} \frac{\sin \thbar_i}{\pi} \prod_{j=i+1}^{d-1} \frac{\pi - \thbar_j}{\pi}, \quad r = \|x\|,  \quad R = \max(r^2,1),
	\end{equation}
	where $\theta_i = g(\bar{\theta}_{i-1})$ with $g$ as in \eqref{eq:def_g}, and observe that $2^d \htilde_x =  (\xi \hat{x}_\star + \zeta \hat{x})$. 
	Let $\alpha := 2^{d} \langle \htilde_x, \hat{x} \rangle$, then we can write:
	\[
	\begin{aligned}
	h_x = \Big[ \frac{\langle x, x\rangle}{2^{2d}} x - \langle \htilde_x , x \rangle \htilde_x \Big] 
	= \frac{r}{2^{2d}} \big[ r^2 \hat{x} - \alpha (\xi \hat{x}_\star + \zeta \hat{x})\big].
	\end{aligned}
	\]
	Using the definition of $\hat{x}$ and $\hat{x}_\star$ we obtain:
	\[
	\frac{2^{2d} h_x}{r} = \big[ (r^2 - \alpha \, \zeta) \cos \thbar_0 - \alpha\, \xi \big] \cdot e_1 + [r^2 - \alpha \zeta] \sin \thbar_0 \cdot e_2,
	\]
	and conclude that since $x \in \Sbeta$, then:
	\begin{align}
	|(r^2 - \alpha \, \zeta) \cos \thbar_0 - \alpha\, \xi | &\leq \beta R  \label{eq:betaMcos}\\
	|[r^2 - \alpha \zeta] \sin \thbar_0| \leq \beta R. \label{eq:betaMsin}
	\end{align}
	We now list some bounds that will be useful in the subsequent analysis. We have:
	\begin{align}
	\thbar_i &\leq \thbar_{i-1} \;\;\text{for}\;\; i \geq 1  \label{eq:th_ilessth_i-1}\\
	\thbar_i &\leq \cos^{-1}(1/\pi) \;\;\text{for}\;\; i \geq 2  \label{eq:thibound}\\
	|\xi| &\leq 1 \label{eq:xibound}\\
	|\zeta| &\leq \frac{d}{\pi} \sin \thbar_0 \label{eq:zetabound}\\
	\xi &\geq \frac{\pi - \thbar_0}{\pi} d^{-3} \label{eq:xilwrbound}\\
	\check{\theta}_i &\leq \frac{3 \pi}{i + 3} \;\;\text{for}\;\; i \geq 0 \label{eq:checkthi1}\\
	\check{\theta}_i &\geq \frac{\pi}{i + 1} \;\;\text{for}\;\; i \geq 0 \label{eq:checkthi2}\\
	\thbar_0 &= \pi + O_1(\delta) \Rightarrow |\xi| \leq \frac{\delta}{\pi} \label{eq:xilrgtheta}\\
	\thbar_0 &= \pi + O_1(\delta) \Rightarrow \zeta = \rho_d + O_1(3 d^3 \delta) \;\text{if}\; \frac{d^2 \delta}{\pi} \leq 1  \label{eq:zetalrgtheta}\\
	1/\pi &\leq \alpha \leq 1 \label{eq:alphabound}.
	\end{align}
	The identities \eqref{eq:th_ilessth_i-1} through \eqref{eq:zetalrgtheta} can be found in Lemma  16 of \cite{HHHV18}, while the identity \eqref{eq:alphabound} follows by noticing that $\alpha = \xi \cos \thbar_0 + \zeta =\cos \theta_d$ and 
	using \eqref{eq:thibound} together with $d\geq 2$.
	\bigskip
	
	\noindent\textbf{Bound on $R$}. We now show that if $x \in \Sbeta$, then $r^2 \leq 4 d$ and therefore $R \leq 4d$.
	
	\noindent If $r^2 \leq 1$, then 
	the claim is trivial. Take $r^2 > 1$, then note that either $|\sin \thbar_0| \geq 1/\sqrt{2}$ or $|\cos \thbar_0| \geq 1/\sqrt{2}$ must hold.
	If $|\sin \thbar_0| \geq 1/\sqrt{2}$ then from \eqref{eq:betaMsin} it follows that $r^2 - \alpha \zeta \leq \sqrt{2} \beta R = \sqrt{2} \beta r^2$ which implies:
	\begin{equation*}
	r^2 \leq \frac{\alpha \,\zeta}{1 - \sqrt{2} \beta}  \leq \frac{1}{(1 - \sqrt{2} \beta)}\frac{d}{\pi} \leq \frac{d}{2}
	\end{equation*}
	using \eqref{eq:zetabound} and \eqref{eq:alphabound} in the second inequality and $\beta < 1/4$ in the third. 
	Next take $|\cos \thbar_0| \geq 1/\sqrt{2}$, then \eqref{eq:betaMcos} implies $|r^2 - \alpha \zeta| \leq \sqrt{2}(\beta r^2 + \alpha \xi)$ which in turn results in:
	\[
	r^2 \leq \frac{\alpha (\zeta + \sqrt{2} \xi)}{1 - \sqrt{2} \beta} \leq 4 d
	\]
	using \eqref{eq:xibound}, \eqref{eq:zetabound}, \eqref{eq:alphabound} and $\beta < 1/4$.  In conclusion if $x \in \Sbeta $  then $r^2 \leq 4 d \Rightarrow R \leq 4 d$.
	\bigskip

	\noindent \textbf{Bounds on $\thbar_0$.} We now show we only have to analyze the small angle case $\thbar_0 \approx 0$ and the large angle case $\thbar_0 \approx \pi$.  
	
	\noindent At least one of the following three cases must hold:
	\begin{enumerate}
		\itemsep1em
		\item ${\sin \thbar_0 \leq 16 \beta \pi d^4}$:  Then we have $\thbar = O_1(32 \pi \beta \pi d^4)$ or $\thbar = \pi + O_1(32 \pi \beta \pi d^4)$ as $32 \pi \beta \pi d^4 <~1$.
		
		\item ${|r^2 - \alpha \zeta| <  \sqrt{\beta}R}$: Then \eqref{eq:betaMcos}, \eqref{eq:alphabound} and $\beta < 1$ yield $|\xi| \leq 2 \sqrt{\beta} \pi R$. Using \eqref{eq:xilwrbound}, we then get \\$\thbar~=~\pi~+~O_1(2 \sqrt{\beta} \pi^2 d^3 R)$.
		
		\item ${\sin \thbar_0 > 16 \beta \pi d^4}$ and $|r^2 - \alpha \zeta| \geq \sqrt{\beta}R$: Then \eqref{eq:betaMsin} gives $|r^2 - \alpha \zeta| \leq \beta M/\sin \thbar_0$ which used with \eqref{eq:betaMcos} leads to:
		\[
		|\alpha \xi| \leq \beta R + |r^2 - \alpha \zeta| \leq  \beta R +  \frac{\beta R}{\sin \thbar_0}  \leq 2 \frac{\beta R}{\sin \thbar_0}.
		\]
		Then using \eqref{eq:alphabound}, the assumption on $\sin \thbar_0$ and $R \leq 4 d$ we obtain $\xi \leq d^{-3}/2$. The latter together with \eqref{eq:xilwrbound} leads to $\thbar_0 \geq \pi/2$. 
		Finally as $|r^2 - \alpha \zeta| \geq \sqrt{\beta}R $ then \eqref{eq:betaMsin} leads to $|\sin \thbar_0| \leq \sqrt{\beta}$. Therefore as $\thbar_0 \geq \pi/2$ and $\beta < 1$, we can conclude that $\thbar_0 = \pi + O_1(2 \sqrt{\beta})$.
	\end{enumerate}
	\bigskip
	
	\noindent Inspecting the three cases, and recalling that $R \leq 4 d$, we can see that it suffices to analyze the small angle case $\thbar_0 = O_1(32 d^4 \pi \beta)$ and the large angle case $\thbar = \pi + O_1(8 \sqrt{\beta}  \pi^2 d^4)$.
	
	\medskip
	\noindent\textbf{Small angle case.}
	We assume $\thbar_0 = O_1(\delta)$ with $\delta = 32 d^4 \pi \beta$ and show that $\|x\|^2 \approx \|\xstar\|^2$.
	
	\noindent We begin collecting some bounds. 
	Since $\thbar_i \leq \thbar_0 \leq \delta$, then $1 \geq \xi \geq (1 - \delta/\pi)^d \geq 1 + O_1(2 d \delta / \pi)$ assuming $\delta d/ \pi \leq 1/2$, which holds true since $64 d^5 \beta < 1$. Moreover from \eqref{eq:zetabound} we have $\zeta = O_1(d \delta/\pi)$.
	Finally observe that $\cos \thbar_0 = 1 + O_1(\thbar_0^2/2) = 1 + O_1(\delta/2)$ for $\delta < 1$.
	We then have $\alpha = 1 + O_1(2 d \delta)$
	so that $\alpha \zeta = O_1(d^2 \delta)$  and $\alpha \xi = 1 + O_1(4 d^2 \delta)$. We can therefore 
	rewrite \eqref{eq:betaMcos} as:
	\[
	(r^2 + O_1(d^2 \delta)) (1 + O_1({\delta}/{2})) - (1 + O_1(4 d^2 \delta)) = O_1(\beta R).
	\]
	Using the bound $r^2 \leq R \leq 4 d$ and the definition of $\delta$, we obtain:
	\begin{equation}\label{eq:small_r2}
	\begin{aligned}
	r^2 - 1 &= O_1\Big(\frac{\delta r^2}{2} + d^2 \delta +  \frac{d^2 \delta^2}{2} + 4 d^2 \delta + 4  d \beta \Big) \\
	&= O_1(8 d^2 \delta + 4 d \beta) \\
	&= O_1(258 \pi d^6 \beta)
	\end{aligned}
	\end{equation}
	\medskip
	
	\noindent\textbf{Large angle case.}	 
	Here we assume $\thbar = \pi + O_1(\delta)$ with $\delta = 8 \sqrt{\beta} \pi^2 d^4$ and show 
	that it must be $\| x \|^2 \approx \rho_d^2  \| \xstar \|^2$.
	
	\noindent 
	From \eqref{eq:xilrgtheta} we know that $\xi = O_1(\delta/\pi)$, while from \eqref{eq:zetalrgtheta} we know that $\zeta = \rho_d +O_1(3 d^3 \delta)$ as long as $8 \sqrt{\beta} \pi d^6 \leq 1$. Moreover for large angles 
	and $\delta < 1$, it holds $\cos \thbar_0 = -1 + O_1((\thbar_0 - \pi)^2/2) = -1 + O_1(\delta^2/2)$. These bounds lead to:
	\[
	\begin{aligned}
	\alpha &= \xi \cos \thbar_0 + \zeta \\
	&= \rho_d + O_1\big( \frac{\delta}{\pi} + \frac{\delta^3}{2 \pi} + 3 d^3 \delta \big)\\
	&= \rho_d + O_1(4 d^3 \delta),
	\end{aligned}
	\]
	and using $\rho_d \leq d$:
	\[
	\begin{aligned}
	\alpha \zeta &= \rho_d^2 + O_1(4 d^3 \delta \rho_d + 3 d^3 \delta \rho_d + 12 d^6 \delta) = \rho_d^2 + O_1(20 d^6 \delta),\\
	\alpha \xi &= O_1(\frac{\delta}{\pi} \rho_d + 4 \frac{d^3 \delta^2}{\pi}) = O_1(2 d^3 \delta).
	\end{aligned}
	\] 
	Then recall that \eqref{eq:betaMcos} is equivalent to $(r^2 - \alpha \zeta) \cos \thbar_0 - \alpha \xi = O_1(4 \beta d)$, that is:
	\[
	\big(r^2 - \rho_d^2 + O_1(20 d^6 \delta) \big) \big(1 + O_1({\delta^2}/{2}) \big) + O_1(2 d^3 \delta) = O_1(4 \beta d)
	\]
	and in particular:
	\begin{equation}\label{eq:r2-rho2}
	\begin{aligned}
	r^2 -\rho_d^2 &= O_1 \Big( 20 d^6 \delta +10 d^6 \delta^3 + \frac{\rho_d \delta^2}{2} + \frac{r^2 \delta^2}{2}  +
	2 d^3 \delta + 4 \beta d \Big) \\
	&= O_1\big( 35 d^6 \delta + 4 \beta d\big) \\
	&= O_1(281 \pi^2 \sqrt{\beta} d^{10})
	\end{aligned}
	\end{equation}
	where we used $\rho_d \leq d$,  the definition of $\delta$ and $\delta < 1$.
	
	\medskip
	\noindent \textbf{Controlling the distance.} We have shown that it is either $\thbar_0 \approx 0$ and $\|x\|^2 \approx \|\xstar\|^2$
	or $\thbar_0 \approx \pi$ and $\|x\|^2 \approx \rho_d^2 \| \xstar \|^2$. We can therefore conclude that 
	it must be either $x \approx \xstar$ or $x \approx -\rho_d \xstar$. 
	
	Observe that if a two dimensional point is known to have magnitude within $\Delta r$ of 
	some $r$ and is known to be within an angle $\Delta \theta$ from 0, then its Euclidean distance to the point of coordinates $(r,0)$ is no more that $\Delta r + (r + \Delta r)\Delta \theta$. Similarly we can write:
	\begin{equation}\label{eq:x1-x2}
	\| x - x_\star \| \leq |\|x\| - \|x_\star \| | + (\|x_\star \| + |\|x\| - \|x_\star \||) \thbar_0.
	\end{equation}
	
	In the small angle case, by \eqref{eq:small_r2}, \eqref{eq:x1-x2}, and $\|\xstar\|\, |\|x\| - \|x_\star \| | \leq |\|x\|^2 - \|x_\star \|^2 |$,   we have: 
	\[
	\|x - \xstar \| \leq  258 \pi d^6 \beta + (1 + 258 \pi d^6 \beta)\, 32 d^4 \pi \beta 
	\leq 550\, \pi d^{10} \beta.
	\]
	
	Next we notice that $\rho_2 = 1/\pi$ and $\rho_{d} \geq \rho_{d-1}$  
	as follows from the definition and \eqref{eq:checkthi1}, \eqref{eq:checkthi2}. Then considering the large angle case and using 
	\eqref{eq:r2-rho2} we have:
	\[
	|\|x\| - \rho_d | \leq \frac{281 \pi^2 \sqrt{\beta} d^{10}}{\|x\| + \rho_d} \leq 281 \pi^3 \sqrt{\beta} d^{10}.
	\]
	The latter, together with \eqref{eq:x1-x2}, yields:
	\[
	\begin{aligned}
	\|x + \rho_d \xstar \| &\leq |\|x\| -  \rho_d | + (\rho_d+ |\|x\| - \rho_d|)(\pi- \thbar_0) \\
	&\leq 281 \pi^3 \sqrt{\beta} d^{10} + (d + 281 \pi^3 \sqrt{\beta} d^{10}) 8 \sqrt{\beta} \pi^2 d^4 \\
	&\leq 284 \pi^3 \sqrt{\beta} d^{10}
	\end{aligned}
	\]
	where in the second inequality we have used $\rho_d \leq d$ and in the third $8 \sqrt{\beta} \pi d^6 \leq 1$.
	
	We conclude by noticing that $\rho_d \to 1$ as $d \to 1$ as shown in \citep[Lemma 16]{HHHV18}. 
\end{proof}

%% file: sections/supplement21.tex
\subsection{Proof of Proposition \ref{prop:local_min}}\label{sec:proof_loc_min}

Recall that $f(x):= 1/4 \|G(x)G(x)^\T - G(\xstar)G(\xstar)^\T - H\|^2_F$, we next define the following loss functions:
\[
\begin{aligned}
f_0(x) &:= \frac{1}{4} \|G(x)G(x)^\T - G(\xstar)G(\xstar)^\T\|^2_F, \\
f_H(x) &:= f_0(x) - \frac{1}{2} \langle G(x)G(x)^\T- G(\xstar)G(\xstar)^\T, H  \rangle_F, \\
f_E(x) &:= \frac{1}{4} \bigg( \frac{1}{2^{2d}}\|x\|^4 + \frac{1}{2^{2d}} \|\xstar\|^4 -2 \langle x, \htilde_x \rangle^2 \bigg).
\end{aligned}
\]
In particular notice that $f(x) = f_H(x) + 1/4 \|H\|^2_F$. Below we show that assuming the WDC is satisfied 
$f_0(x)$ concentrates around $f_E(x)$. 

\begin{lemma}\label{lemma:f0_conctr}
	Suppose that $d\geq 2$ and the WDC holds with $\eps < 1/(16 \pi d^2)^2$, then for all nonzero $x, \xstar \in \R^k$
	\[
	|f_0(x) - f_E(x)| \leq	\frac{16}{2^{2d}}  (\|x\|^4  + \|\xstar\|^4) d^4 \sqrt{\eps}
	\]
\end{lemma}

We next consider the loss $f_E$ and show that 
in a neighborhood $-\rho_d \xstar$, this loss function has larger values than 
in a neighborhood of $ \xstar$.

\begin{lemma}\label{lemma:UPandLOW}
	Fix $0 < a \leq 1/ (2\pi^3 d^3)$ and $\phi_d \in [\rho_d , 1]$ then:
	\[
	\begin{aligned}
	f_E(x) &\leq \frac{1}{2^{2d+2}} \|\xstar\|^4 + \frac{1}{2^{2d+2}}  \big[ (a + \phi_d)^4 - 2 \phi_d^2 + 2 \pi d a \big] \|\xstar\|^4 \quad \forall x \in \mathcal{B}(\phi_d \xstar, a \|\xstar\|) \; \text{and}\\
	f_E(x) &\geq \frac{1}{2^{2d+2}} \|\xstar\|^4 + \frac{1}{2^{2d+2}}  \big[ (a - \phi_d)^4 - 2 \rho_d^2 \phi_d^2 - 40 \pi d^3 a \big] \|\xstar\|^4 \quad \forall x \in \mathcal{B}(- \phi_d \xstar, a \|\xstar\|).\\
	\end{aligned}
	\]
\end{lemma}

The above two lemmas are now used to prove Proposition \ref{prop:local_min}.
\begin{proof}[Proof of Proposition \ref{prop:local_min}]
	Let $x \in \mathcal{B}(\pm \phi_d \,\xstar, \varphi \|\xstar\|)$ for a $0 < \varphi < 1$ that will be specified below, and observe that by the assumptions on the noise:
	\[
	\begin{aligned}
	|\langle G(x) G(x)^\T - G(\xstar)G(\xstar)^\T, H \rangle_F| &\leq |G(x)^\T H G(x)| + |G(\xstar)^\T H G(\xstar)| \\
	&\leq \frac{\omega}{2^d} (\|x\|^2 + \|\xstar\|^2) \\
	&\leq \frac{\omega}{2^d} ((\phi_d + \varphi)^2 + 1) \| \xstar \|^2,
	\end{aligned}
	\] 
	and therefore by Lemma \ref{lemma:f0_conctr}:
	\[
	\begin{aligned}
	|f_0(x) - f_E(x)| &+  \frac{1}{2}|\langle G(x) G(x)^\T - G(\xstar)G(\xstar)^\T, H \rangle_F| \leq \\
	&\leq \frac{16}{2^{2d}} ((\phi_d + \varphi)^4 + 1) \|\xstar\|^4 d^4 \sqrt{\eps} + \frac{\omega}{2^d} ((\phi_d + \varphi)^2 + 1) \| \xstar \|^2 \\
	&\leq \frac{272}{2^{2d}}\|\xstar\|^4 d^4 \sqrt{\eps} + \frac{\omega}{2^d} ((\phi_d + \varphi)^2 + 1) \| \xstar \|^2
	\end{aligned}
	\]  
	We next take $\varphi = \eps$ and $x \in \mathcal{B}(\phi_d \,\xstar, \varphi \|\xstar\|)$, so that by Lemma \ref{lemma:UPandLOW} and the assumption $2^d d^{12} w \leq K_2 \| \xstar \|^2$, we have:
	\[
	\begin{aligned}
	f_H(x) &\leq f_E(x) + |f_0(x) - f_E(x)| + \frac{1}{2}  |\langle G(x) G(x)^\T - G(\xstar)G(\xstar)^\T, H \rangle_F | \\
	&\leq \frac{1}{2^{2d+2}} \big[ 1 + (\eps + \phi_d)^4 - 2 \phi_d^2 + 2 \pi d \eps  \big] \|\xstar \|^4 
	+ 272 d^4 \sqrt{\eps} \|\xstar\|^4
	+ \frac{\omega}{2^{d+1}} (2 + 2 \eps + \eps^2) \|\xstar\|^2 \\
	&\leq \frac{1}{2^{2d +2}} \big[ 1 - 2 \phi_d^2 + (\eps + \phi_d)^4 \big]  \|\xstar\|^4 
	+\frac{1}{2^{2d}}\Big(\frac{3}{2} 2^d \|\xstar\|^{-2} \omega +  \frac{\pi d}{2} + 272 d^4 \Big)\sqrt{\eps} \|\xstar\|^4  + \frac{\omega }{2^{d}}\|\xstar\|^2\\
	&\leq \frac{1}{2^{2d +2}} \big[ 1 - 2 \phi_d^2 + (\eps + \phi_d)^4 \big]  \|\xstar\|^4 
	+\frac{1}{2^{2d}}\Big(\frac{3}{2} K_2 d^{-12} +  \frac{\pi d}{2} + 272 d^4 \Big)\sqrt{\eps} \|\xstar\|^4  + K_2 \frac{\|\xstar\|^4}{2^{2d}} d^{-12}.
	\end{aligned}
	\]
	Similarly if $y \in \mathcal{B}(- \phi_d \,\xstar, \varphi \|\xstar\|)$, and $\varphi = \eps$ we obtain:
	\[
	\begin{aligned}
	f_H(y) &\geq f_E(y) - |f_0(y) - f_E(y)| - \frac{1}{2}  |\langle G(y) G(y)^\T - G(\xstar)G(\xstar)^\T, H \rangle | \\
	&\geq \frac{1}{2^{2d +2}} \big[ 1 - 2 \phi_d^2 \rho_d^2 + (\eps - \phi_d)^4 \big]  \|\xstar\|^4 
	- \frac{1}{2^{2d}} \Big( \frac{3}{2} 2^d \|\xstar\|^{-2} \omega +  10 \pi d^3 + 272 d^4 \Big)\sqrt{\eps} \|\xstar\|^4  - \frac{\omega }{2^{d}}\|\xstar\|^2
	\\&\geq \frac{1}{2^{2d +2}} \big[ 1 - 2 \phi_d^2 \rho_d^2 + (\eps - \phi_d)^4 \big]  \|\xstar\|^4 
	- \frac{1}{2^{2d}} \Big( \frac{3}{2} K_2 d^{-12} +  10 \pi d^3 + 272 d^4 \Big)\sqrt{\eps} \|\xstar\|^4  - K_2 \frac{\|\xstar\|^4}{2^{2d}} d^{-12}.
	\end{aligned}
	\]
	In order to guarantee that $f(y) > f(x)$, it suffices to have:
	\[
	2 (1 - \rho_d^2) \phi_d^2 -  8 K_2 d^{-12} > 4 C_d \sqrt{\eps}
	\]
	with $C_d := (544 d^4 + 10 \pi d^3 \pi + 3 K_2 d^{-12} + \pi d/2 + 1/100)$,
	that is to
	require:
	\[
	\varphi = {\eps} < \bigg( \frac{(1 - \rho_d^2) \phi_d^2 - 4 K_2 d^{-12}}{2 C_d} \bigg)^2.
	\]
	Finally notice that by Lemma 17 in \cite{HHHV18} it holds that $1- \rho_d \geq (K (d +2))^{-2}$ for some numerical constant $K$, we  therefore choose 
	$\eps = \varrho/d^{12}$ for some $\varrho >0$ small enough.
\end{proof}

%% file: sections/supplement3.tex
\subsection{Supplementary proofs}\label{sec:supp_proof}


Below we prove Lemma \ref{lemma:conctr_grad1} on the concentration of the gradient of $f$ at a differentiable point. 
\begin{proof}[Proof of Lemma \ref{lemma:conctr_grad1}]
	We begin by noticing that:
		\[\vbar_x - h_x = \big[\langle p_x, x \rangle p_x - \langle x, x\rangle \frac{x}{2^{2d}} \big] + [ \langle \htilde_x , x \rangle x -  \langle q_x , x \rangle x ].\]
		Below we show  that:
		\begin{equation}\label{eq:px-x}
		\| \langle p_x, x \rangle p_x - \langle x, x\rangle \frac{x}{2^{2d}}\| \leq \frac{50}{2^{2d}} d^3 \sqrt{\eps} \max\{\|x\|^2, \|\xstar\|^2\} \|x\|.
		\end{equation}
		and
		\begin{equation}\label{eq:qx-htx}
		\| \langle q_x, x \rangle p_x - \langle \htilde_x, x\rangle \htilde_x| \leq \frac{36}{2^{2d}} d^4 \sqrt{\eps} \max\{\|x\|^2, \|\xstar\|^2\} \|x\|.
		\end{equation}
		from which the thesis follows.

		Regarding equation \eqref{eq:px-x} observe that:
		\[
			\begin{aligned}
				\| \langle p_x, x \rangle p_x - \langle x, x\rangle \frac{x}{2^{2d}}\| &= \| \langle p_x ,x \rangle  \big[ p_x - \frac{x}{2^d} \big] +  \langle p_x - \frac{x}{2^d} , \frac{x}{2^d} \rangle x \| \\				
				&\leq  \big( \|\Lambda_x x\|^2 + \frac{\|x\|^2}{2^{d}} \big) \|p_x - \frac{x}{2^d}\| \\
				& \leq \frac{50}{2^{2d}} d^3 \sqrt{\eps} \|x\|^3
			\end{aligned}
		\]
		where in the first inequality we used $\langle p_x, x\rangle = \|\Lambda_x x\|^2$ and in the second we used equations \eqref{eq:htilde_conctr} and \eqref{eq:Lx_bound} of Lemma \ref{lemma:conctrWDC}.
		
		Next note that:
		\[
			\begin{aligned}
			\| \langle q_x, x \rangle q_x - \langle \htilde_x, x\rangle \htilde_x \| &= \| \langle q_x ,x \rangle  (q_x - \htilde_x) +  \langle q_x - \htilde_x , x \rangle \htilde_x \| \\				
			&\leq  ( \|q_x\| + \|\htilde_x\| ) \| x\| \|q_x - \htilde_x\| \\
			& \leq \big(\frac{13}{12} + 1 + \frac{d}{\pi}\big)  \frac{\|x\| \| \xstar\|}{2^d} \|q_x - \htilde_x\|\\
			&\leq \frac{3}{2} d \frac{\|x\| \| \xstar\|}{2^d} \|q_x - \htilde_x\|
			\end{aligned}
		\]
		where in the second inequality we have the bound
		\eqref{eq:Lx_bound} and the definition of $\htilde_{x}$. Equation \eqref{eq:qx-htx} is then 
		found by appealing to equation \eqref{eq:htilde_conctr} in Lemma \ref{lemma:conctrWDC}.
\end{proof}

The previous lemma is now used to control the concentration 
of the subgradients $v_x$ of $f$ around $h_x$.

\begin{proof}[Proof of Lemma \ref{lemma:conctr_grad2}]
	
	When $f$ is differentiable at $x$, $\nabla f(x) = \vtilde_x = \vbar_x + \eta_x$, so that by
	Lemma \ref{lemma:conctr_grad1} and the assumption on the noise:
	\begin{equation}\label{eq:gradwnoise}
	\begin{aligned}
		\| v_x - h_x \| &\leq \| \vbar_x
		- h_x\| + \|\eta_x\| \\
		&\leq 86 \frac{ d^4 \sqrt{\eps}}{2^{2d}} \max(\|\xstar\|^2, \| x\|^2) \|x\| +   \frac{\omega}{2^{d}} \,  \| x\|.
	\end{aligned}
	\end{equation}
	Observe, now, that by \eqref{eq:subdiff}, for any $x \in \R^k$,  $v_x \in \pa f(x) = $conv$(v_1, \dots, v_t)$, and therefore $v_x = a_1 v_1 + \dots + a_T v_T$ for some $a_1, \dots, a_T  \geq 0$, $\sum_i a_i = 1$.
	Moreover for each $v_i$ there exist a $w_i$ such that $v_i = \lim_{\delta \to 0^+} \vtilde_{x + \delta w_i} $. Therefore using equation \eqref{eq:gradwnoise}, the continuity of $h_x$ with respect to nonzero $x$ and $\sum_i a_i = 1$:
	\[
	\begin{aligned}
	\| v_x - h_x \| &\leq \sum_{i=1}^T a_i \| v_i  - h_x\| \\
	& \leq  \sum_{i=1}^T a_i \lim_{\delta \to 0} \|   \vtilde_{x + \delta w_i}  - h_{x + \delta w_i}\| \\
	&\leq 86 \frac{ d^4 \sqrt{\eps}}{2^{2d}} \max(\|\xstar\|^2, \| x\|^2) \|x\| + \frac{\omega}{2^{d}}  \,  \| x\|.
	\end{aligned}
	\]
\end{proof}

We now prove Lemma \ref{lemma:f0_conctr} on the concentration of 
the noiseless objective function.

\begin{proof}[Proof of Lemma \ref{lemma:f0_conctr}]
	Observe that:
	\[
	\begin{aligned}
	|f_0(x) - f_E(x)| &\leq \frac{1}{4} |\|G(x)\|^4 - \frac{1}{2^{2d}} \|x\|^4 | \\ 
	&+\frac{1}{4}  |\|G(\xstar)\|^4 - \frac{1}{2^{2d}} \|\xstar\|^4 | \\ 
	&+ \frac{1}{2} |\langle G(x),G(\xstar) \rangle^2 - \langle x, \htilde_x \rangle |.
	\end{aligned}
	\]
	We analyze each term separately. 
	The first term can be bounded as:
	\begin{equation*}
	\begin{aligned}
	\frac{1}{4} |\|G(x)\|^4 - \frac{1}{2^{2d}} \|x\|^4 | &=  \frac{1}{4} |\|G(x)\|^2 + \frac{1}{2^d} \|x\|^2 | \;\:  | \|G(x)\|^2 - \frac{1}{2^d} \|x\|^2 | \\ 
	&\leq  \frac{1}{4}  \frac{1}{2^d} \big( \frac{13}{12} + 1 \big) \|x\|^2  \;\:  | \|G(x)\|^2 - \frac{1}{2^d} \|x\|^2 | \\
	&\leq  \frac{1}{4}  \frac{1}{2^d} \Big( \frac{13}{12} + 1 \Big) \|x\|^2  \; \: 24 \frac{d^3 \sqrt{\eps}}{2^d} \|x\|^2  \\	
	&\leq \frac{1}{2^{2d}} 13 d^3\, \sqrt{\eps} \|x\|^4	
	\end{aligned}
	\end{equation*}
	where in the first inequality we used $\eqref{eq:Lx_bound}$ and in the second inequality \eqref{eq:htilde_conctr} . 
	Similarly we can bound the second term:
	\[
	\frac{1}{4} |\|G(\xstar)\|^4 - \frac{1}{2^{2d}} \|\xstar\|^4 |  \leq \frac{1}{2^{2d}} 13 d^3\, \sqrt{\eps} \|\xstar\|^4.
	\] 
	
	We next note that $\|\htilde_x \| \leq 2^{-d} (1 + d/\pi) \|\xstar\|$ and therefore
	from \eqref{eq:Lx_bound} and $d\geq 2$ we obtain:
	\[
	|\|G(x)\| \|G(\xstar)\| + \|x\| \|\htilde_x\| | \leq \frac{1}{2^d}	\big(\frac{13}{12} + 1 + \frac{d}{\pi}	\big) \|x\| \|\xstar\| \leq \frac{1}{2^d} \frac{3}{2} d \|x\| \|\xstar\|	
	\]
	We can then conclude that:
	\[
	\begin{aligned}
	\frac{1}{2} |\langle G(x),G(\xstar) \rangle^2 - \langle x, \htilde_x \rangle^2 | 
	&\leq	\frac{1}{2}  |\langle x, \Lambda_x^T \Lambda_\xstar \xstar - \htilde_x \rangle| \: \; |\|G(x)\| \|G(\xstar)\| + \|x\| \|\htilde_x\| | \\
	&\leq 	\frac{1}{2} \|x\| 24 \frac{d^3 \sqrt{\eps}}{2^d} \|\xstar\| \: \;
	\frac{1}{2^d} \frac{3}{2} d \; \|x\| \|\xstar\|	\\ 
	&\leq \frac{9}{2^{2d}} d^4 \sqrt{\eps}  (\|\xstar\|^4 + \|x\|^4) 
	\end{aligned}.
	\]
\end{proof}

Below we prove lower and upper bound on the loss $f_E$ as in Lemma \ref{lemma:UPandLOW}.

\begin{proof}[Proof of Lemma \ref{lemma:UPandLOW}]
	Let $x \in \mathcal{B}(\phi_d \xstar, a \|\xstar\|)$ then observe that $0 \leq \thbar_i \leq \thbar_0 \leq \pi a / 2 \phi_d$ and $(\phi_d - a) \|\xstar\| \leq  \|x \| \leq (a + \phi_d) \|\xstar\|$. 
	Then observe that:
	\[
	\begin{aligned}
	\langle x, \htilde_d \rangle &= \frac{1}{2^d}  \big(\prod_{i=0}^{d-1} \frac{\pi - \thbar_i}{\pi}\big) \|\xstar\| \| x\| \cos \thbar_0  + \frac{1}{2^d} \sum_{i=0}^{d-1} \frac{\sin \thbar_i}{\pi} \prod_{j=i+1}^{d-1} \frac{\pi - \thbar_j}{\pi} \:  \|\xstar\| \| x\| \\
	&\geq \frac{1}{2^d}  \big(\prod_{i=0}^{d-1} \frac{\pi -\frac{ \pi a}{2 \phi_d}}{\pi} \big) \:  (\phi_d - a)   \| \xstar \|^2   \big( 1 - \frac{\pi^2 a^2}{8 \phi_d^2}\big) \\
	&\geq \frac{1}{2^d}  \big( 1- \frac{d a}{\phi_d} \big)  (\phi_d - a)  \big( 1 - \frac{\pi^2 a^2}{8 \phi_d^2}\big) \| \xstar \|^2.
	\end{aligned}
	\] 
	using  $\cos \theta \geq 1 - \theta^2/2$ and $(1 - x)^d \geq (1 - 2 d x)$ as long as $0 \leq x \leq 1$. We can therefore write:
	\[
	\begin{aligned}
	f_E(x) - \frac{\|\xstar\|^4}{2^{2d+2}} &\leq \frac{1}{2^{2d+2}} \|x\|^4 - 
	\frac{1}{2^{2d+1}}  \big( 1- \frac{d a}{\phi_d} \big)^{2}  (\phi_d - a)^{2}  \big( 1 - \frac{\pi^2 a^2}{8 \phi_d^2}\big)^{2} \| \xstar \|^4 \\
	&\leq \frac{1}{2^{2d+2}} \Big[  (\phi_d + a)^4 - 2  \big( 1- 2 \frac{d a}{\phi_d} \big)  (\phi_d - a)^{2}  \big( 1 - \frac{\pi^2 a^2}{4 \phi_d^2}\big) \Big]  \| \xstar \|^4
	\end{aligned}
	\]
	where in the second inequality we used $(1 - x)^2 \geq 1 - 2x\;$ for all $x \in \R$. We then observe that:
	\[
	\begin{aligned}
		\big( 1- 2 \frac{d a}{\phi_d} \big)  (\phi_d - a)^{2}  \big( 1 - \frac{\pi^2 a^2}{4 \phi_d^2}\big) &\geq \big( 1 -\frac{\pi ^2 a^2}{4 \phi_d ^2}-\frac{2 a d}{\phi_d }\big) \phi_d^2 + a (a - 2 \phi_d)	\big( 1- 2 \frac{d a}{\phi_d} \big)   \big( 1 - \frac{\pi^2 a^2}{4 \phi_d^2}\big) \\
		&\geq \phi_d^2 -  a \big( \frac{1}{2 \pi d^3} +  {2 d \phi_d} \big) + a (a - 2 \phi_d) \big( 1- 2 \frac{d a}{\phi_d} \big)   \big( 1 - \frac{\pi^2 a^2}{4 \phi_d^2}\big)  \\
		&\geq \phi_d^2 -  a \big( \frac{1}{2 \pi d^3} +  {2 d \phi_d} + 2 \phi_d  \big) \\
		&\geq \phi_d^2 - \pi d a,
	\end{aligned}
	\]
	where in the second inequality we have used $\pi^3 d^3 a \leq 2$ and in the last  one $d \geq 2$ and $\phi_d \leq 1$.
	We can then conclude that:
	\[
	f_E(x) - \frac{\|\xstar\|^4}{2^{2d+2}}  \leq \frac{1}{2^{2d+2}} \big[ (\phi_d + a)^4 - 2  (\phi_d^2 - \pi d a) \big]\|\xstar\|^4 
	\]
	
	We next take $x \in \mathcal{B}(- \phi_d \xstar, a \|\xstar\|)$ which implies $0 \leq \pi - \thbar_0 \leq \pi^2 a /2 =: \delta $ and $\|x\| \leq (a + \phi_d) \|\xstar\|$. We then note that for 
	$\xi$ and $\zeta$ as defined in \eqref{eq:zeta_xi_def} we have:
	\[
	\begin{aligned}
	|2^d x^\T \htilde_x|^2 &\leq (|\xi| + |\zeta|)^2 (a + \phi_d)^2 \|\xstar\|^4 \\
	&\leq \big(\frac{\delta}{\pi} + 3 d^3 \delta + \rho_d \big)^2 (a + \phi_d)^2 \|\xstar\|^4 \\ 
	&\leq \big( \frac{\pi^3 d^3}{2} a+ \rho_d \big)^2 (a + \phi)^2 \|\xstar\|^4 \\
	&\leq ( 2 \pi^3 d^3 a + \rho_d^2 ) (a + \phi_d)^2 \|\xstar\|^4 \\
	&\leq 20 \pi d^3 a + \rho_d^2 \phi_d^2		
	\end{aligned}
	\]
	where the second inequality is due to \eqref{eq:xilrgtheta} and \eqref{eq:zetalrgtheta}, the rest  from $d \geq 2$, $ \rho_d \leq \phi_d \leq 1$ and  $2 \pi^3 d^3 a \leq 1$.
	Finally using $(\phi_d - a) \|\xstar\|\leq \|x\|$, we can then conclude that:
	\[
	f_E(x) - \frac{\|\xstar\|^4}{2^{2d+2}}  \geq \frac{1}{2^{2d+2}} \big[ (\phi_d - a)^4 - 2 (20 \pi d^3 a + \rho_d^2 \phi_d^2) \big] \|\xstar\|^4.
	\]  
\end{proof}

%% file: sections/supplement4.tex
\section{Proofs for the random spiked and generative models}\label{sec:rand_proofs}

We are now ready to prove our main results for random spiked models and generative networks with random weights. We begin by recalling the following fact on the WDC of a single Gaussian layer. 

\begin{lemma}[Lemma 11 in \cite{HV17}]\label{lemma:rndWDC}
	Fix $0 < \eps < 1$ and suppose $W \in \R^{n \times k}$ has $i.i.d. \;\mathcal{N}(0,1/n)$ entries. Then if $n \geq C_\eps k \log k$, then with probability at least $1 - 8 n \exp(-\gamma_\eps k)$, $W$ satisfies the WDC with constant $\eps$. Here $C_\eps$ and $\gamma_\eps^{-1}$ depend 
	polynomially on $\eps^{-1}$.
\end{lemma}

By a union bound over all layers, using the above result we can 
conclude that the WDC holds simultaneously for all layers of the network with probability at least
$    1  - \sum_{i=1}^d 8 n_i e^{- \gamma_\eps n_{i-1}}.$
Note in particular that this argument does not requires the independence of the layers.

By Lemma \ref{lemma:rndWDC}, with high probability 
the  random generative network $G$ satisfies the WDC. Therefore if we can guarantee the assumptions on the noise term, then the proof of the main Theorem \ref{thm:main_rand}
follows from the deterministic Theorem \ref{thm:dir_deriv} and the previous lemma. 

Before turning to the bounds of the noise terms in the spiked models,
we recall the following lemma which bounds the number of possible $\Lambda_x$ for $x\neq 0$. Note that this is related to the number of possible regions defined by a deep Relu network.

\begin{lemma}[Proof of Lemma 8 in \cite{HHHV18}]\label{lemma:NLx}
	Consider a network $G$ as defined in \eqref{eq:Gx} with $d \geq 2$, weight matrices $W_i \in \R^{n_i \times n_{i-1}}$ with i.i.d. entries $\mathcal{N}(0,1/n_i)$ and $\log(10) \leq k/4 \log(n_1)$. Then, with probability one, for any $x \neq 0$ the number of different matrices $\Lambda_x$ is:
	\[
	| \{ \Lambda_x | x \neq 0 \} | \leq 10^{d^2} (n_1^d n_2^{d-1} \dots n_d )^k \leq (n_1^d n_2^{d-1} \dots n_d)^{2 k}
	\] 
\end{lemma}

In the next section we use this lemma to control the noise term $\Lambda_x^\T H \Lambda_x$ where:
\begin{itemize}
	\item in the \textbf{Spiked Wishart Model} $H~=~\Sigma_N~-~\Sigma$;
	\item in the \textbf{Spiked Wigner Model} $H~=\mathcal{H}$.
\end{itemize}

We then conclude in section \ref{sec:min_rand} with the proof of Proposition \ref{prop:rand_localmin}.B.

\subsection{Spiked Wigner Model}

Recall that in the Wigner model $Y = G(\xstar) G(\xstar)^\T + \mathcal{H}$ and  the symmetric noise matrix $\mathcal{H}$ follows a \textit{Gaussian Orthogonal Ensemble} GOE$(\nu, n)$, that is $\mathcal{H}_{ii} \sim \mathcal{N}(0, 2\nu/n)$ for all $1 \leq i \leq n$ and $\mathcal{H}_{ij} \sim \mathcal{N}(0,\nu/n)$ for $1 \leq j < i \leq n$. 
Our goal is to bound $ \| \Lambda_x^\T \mathcal{H} \Lambda_x \| $ uniformly over $x$ with high probability.

\noindent Fix $x \in \R^{k}$, 
and let $\mathcal{N}_{{1}/{4}}$ be a $1/4$-net on the sphere $\mathcal{S}^{k-1}$ such that $|\mathcal{N}_{1/4}| \leq 9^k$ and:
\[
\|\Lambda_x^\T \mathcal{H} \Lambda_x \| \leq 2 \max_{z \in \mathcal{N}_{{1}/{4}}} |\langle \Lambda_x^\T \mathcal{H} \Lambda_x z,   z\rangle|.
\]
For any  $z \in\mathcal{N}_{{1}/{4}}$ let $\ell_{x,z} := \Lambda_x z \in \R^n$ and note that by the assumption on the  entries of $\mathcal{H}$ it holds that
$\ell_{x,z}^\T \mathcal{H} \ell_{x,z} \sim  \mathcal{N}(0,  {\nu^2 \|\ell_{x,z}\|^4}/{n})$.
In particular by Lemma \ref{lemma:conctrWDC}, the quadratic form $\ell_{x,z}^\T \mathcal{H} \ell_{x,z}$ is sub-Gaussian with parameter $\gamma^2$ given by: 
\[{\gamma}^2 := \frac{\nu^2}{n} \Big(\frac{13}{12}\Big)^2 \frac{1}{2^{2d}}.\]
Then for fixed $x \in \R^k$, standard sub-Gaussian tail bounds and a union bound over $\mathcal{N}_{{1}/{4}}$ give:
\[
\begin{aligned}
\PX\big[ \| \Lambda_x^\T \mathcal{H}  \Lambda_x \| \geq 2 u \big]
&\leq \PX\big[ \max_{z \in \mathcal{N}_{1/4}} \| \ell_{x,z}^\T \mathcal{H} \ell_{x,z}\| \geq u \big] \\
&\leq \sum_{z \in \mathcal{N}_{1/4}} \PX\big[  \| \ell_{x,z}^\T \mathcal{H}  \ell_{x,z}\| \geq u \big]
\leq 2 \cdot 9^{k} e^{-\frac{u^2}{2 {\gamma}^2}}.
\end{aligned}
\]
Lemma \ref{lemma:NLx}, then ensures that the number of possible $\Lambda_x$ is at most $(n_1^d n_2^{d-1} \dots n_d)^{2 k}$, so a union bound over this set allows to conclude that:
\[
\PX\big[ \| \Lambda_x^\T \mathcal{H}  \Lambda_x \| \leq  \frac{\nu}{2^d} \sqrt{  \frac{30 k \log (3\, n_1^d n_2^{d-1} \dots n_d)}{n}}, \; \text{for all}\, x \big] 		
\geq 1 - 2 e^{- k \log (n)}.
\]

\subsection{Spiked Wishart Model}

Recall that the data $\{y_i \}_{i=1}^N$ are i.i.d. samples from $\mathcal{N}(0, \Sigma)$ where $\Sigma = G(\xstar)G(\xstar)^\T + \sigma^2 I_n$.
In the minimization problem \eqref{eq:minM} we take $Y = \Sigma_N - \sigma^2 I_n$ where $\Sigma_N$ is the empirical covariance matrix. The symmetric noise matrix $H$ is then given by $H = \Sigma_N - \Sigma$ and by the Law of Large Numbers $H \to 0$ as $N \to \infty$.
We bound $\|\Lambda_x^\T H \Lambda_x \|$ with high probability uniformly over $x \in \R^k$.

Fix $x \in \R^k$, let $\mathcal{N}_{{1}/{4}}$ be a $1/4$-net on the sphere $\mathcal{S}^{k-1}$ such that $|\mathcal{N}_{1/4}| \leq 9^k$, and notice that:
\[
\|\Lambda_x^\T H \Lambda_x \| \leq 2 \max_{z \in \mathcal{N}_{{1}/{4}}} |z^\T \Lambda_x^\T H \Lambda_x   z|.
\]
By a union bound on $\mathcal{N}_{1/4}$ we obtain for any fixed $z \in \mathcal{N}_{{1}/{4}}$:
\[
\PX\big[\|\Lambda_x^\T H \Lambda_x\| \geq 2 u \big] \leq 9^k \PX\big[|z^T \Lambda_x^\T H \Lambda_x z | \geq u\big].
\]
Let $\ell_x := \Lambda_x z$ and note that:
\[
\begin{aligned}
z^T \Lambda_x^\T H \Lambda_x z &= \frac{1}{N} \sum_{i=1}^N (\ell_{x}^\T y_i)^2 - \EX[(\ell_{x}^\T y_i)^2 ]  
\end{aligned}
\]
Since $s_i := \ell_{x}^\T y_i \sim \mathcal{N}(0,\gamma^2)$ where $\gamma^2 = \ell_{x}^\T \Sigma \ell_{x}$, then
for $u/\gamma^2 \in (0,1)$ by small deviation bounds for $\chi^2$ random variables (see for example \cite[Example 2.11]{wainwright2019high}): 
\[ 
\begin{aligned}
\PX\big[\|\Lambda_x^\T H \Lambda_x\| \geq 2 u \big] \leq 9^k \PX\Big[|\frac{1}{N}\sum_{i=1} ({s_i}/{\gamma})^2 - 1| \geq \frac{u}{\gamma^2}\Big]  
\leq 2 \exp\big[ k \log 9 - \frac{N}{8}\frac{u^2}{\gamma^4} \big].
\end{aligned}
\]
Recall now that $|\{\Lambda_x | x \neq 0 \}| \leq (n_1^d n_2^d \dots n_d )^k$, then
proceeding as for the Wigner case by a union bound over all possible $\Lambda_x$:
\[
\PX\big[ \| \Lambda_x^\T H \Lambda_x \| \leq 2 \sqrt{\frac{24 k \log (3\, n_1^d n_2^{d-1} \dots n_d) }{N}} \gamma^2, \;\text{for all}\; x  \big] \geq 1 - 2 e^{- k \log(3 n)} 
\]
Similarly when $u/\gamma^2 \geq 1$ we obtain 
\[
\PX\big[ \| \Lambda_x^\T H \Lambda_x \| \leq 2 {\frac{24 k \log (3\, n_1^d n_2^{d-1} \dots n_d) }{N}} \gamma^2, \;\text{for all}\; x  \big] \geq 1 - 2 e^{- k \log(3 n)} 
\]
\subsection{Proof of Proposition \ref{prop:rand_localmin}}\label{sec:min_rand}

Observe that Proposition \ref{prop:rand_localmin}.A follows from Proposition \ref{prop:local_min} after noticing that the assumptions on $\eps$ and $\omega$ in Theorem \ref{thm:main_rand} imply that $ \mathcal{B}(\xstar, r_+) \subset \mathcal{B}(\xstar, \varrho \|\xstar\| d^{-12})$ and $\mathcal{B}(- \rho_d \, \xstar, r_-) \subset  \mathcal{B}(-\rho_d \xstar, \varrho \|\xstar\| d^{-12})$. 

We next recall the following fact on the local Lipschitz property of the generative network.

\begin{lemma}[Lemma 21 in \cite{HHHV18}]
	Suppose $x \in \mathcal{B}(\xstar, d \sqrt{\eps} \|\xstar\|)$, and the WDC holds with $\eps < 1/(200)^4 /d^6$. Then it holds that:
	\[
	\|G(x) - G(\xstar) \|\leq \frac{1.2}{2^{d/2}} \|x - \xstar\|.
	\]
\end{lemma}

The proof of Proposition \ref{prop:rand_localmin}.B follows now from $\ystar = G(\xstar)$, the above Lemma, the bounds \eqref{eq:LxLy} and \eqref{eq:Lx_bound} and the assumptions on $\epsilon$ and the noise term.